%% file: main.tex
\documentclass{article}
\PassOptionsToPackage{table}{xcolor}
\usepackage{main,times}

\input{math_commands.tex}

\usepackage{times}
\usepackage[T1]{fontenc}
\usepackage[utf8]{inputenc}
\usepackage{amsmath,amssymb,amsthm}
\usepackage{booktabs}
\usepackage{tabularx}
\usepackage{microtype}
\usepackage{graphicx}
\usepackage{float}
\usepackage{placeins}
\usepackage[table]{xcolor}
\usepackage{tikz}
\usetikzlibrary{arrows.meta,positioning}
\usepackage{wrapstuff}
\usepackage{wrapfig}
\usepackage{subcaption}
\usepackage{caption}
\usepackage[colorlinks]{hyperref}
\usepackage{etoc}
\hypersetup{citecolor=[HTML]{3366CC}} 
\usepackage{url}
\usepackage{thmtools}

\usepackage{algorithm}
\usepackage{algpseudocode}
\usepackage[T1]{fontenc}
\usepackage{times,booktabs,graphicx}
\definecolor{knowledgeblue}{HTML}{527DB8}
\definecolor{modelred}{HTML}{C96F69}

\algrenewcommand{\algorithmicrequire}{\textbf{Input:}}
\newcommand{\method}{\textsc{SPLICE}}

\definecolor{KnowledgeBlue}{HTML}{315B8A}
\definecolor{ModelGreen}{HTML}{4F7A5B}

\usepackage[most]{tcolorbox}

\definecolor{propbg}{HTML}{F4F7FA}
\definecolor{propframe}{HTML}{A5B6C8}

\newtcolorbox{propositionbox}{
    enhanced,
    breakable,
    colback=propbg,
    colframe=propframe,
    boxrule=0.5pt,
    arc=1mm,
    boxsep=0pt,
    left=3mm,
    right=3mm,
    top=2mm,
    bottom=2mm,
    before skip=6pt,
    after skip=6pt,
}

\usepackage{listings}
\definecolor{promptframe}{HTML}{3F5F8F}
\definecolor{promptback}{HTML}{F5F8FC}
\lstdefinestyle{promptjson}{
    basicstyle=\ttfamily\footnotesize\color{black},
    columns=fullflexible,
    keepspaces=true,
    breaklines=true,
    breakatwhitespace=true,
    breakindent=1.5em,
    showstringspaces=false,
    aboveskip=3pt,
    belowskip=1pt,
    xleftmargin=0.8em,
}
\newtcolorbox[auto counter]{promptbox}[2][]{
    enhanced,
    colback=promptback,
    colframe=promptframe,
    coltitle=white,
    coltext=black,
    boxrule=0.5pt,
    arc=1.2mm,
    left=3mm,
    right=3mm,
    top=2mm,
    bottom=2mm,
    toptitle=0.8mm,
    bottomtitle=0.8mm,
    fonttitle=\small\bfseries,
    title={Prompt~\thetcbcounter: #2},
    fontupper=\linespread{1.06}\small,
    before upper={\setlength{\parindent}{0pt}},
    before skip=8pt,
    after skip=10pt,
    #1
}
\newcommand{\promptsec}[1]{\par\addvspace{2pt}{\bfseries #1}\par\nobreak}
\newenvironment{promptitems}{%
    \begin{list}{\textbullet}{%
        \setlength{\leftmargin}{1.3em}%
        \setlength{\labelwidth}{0.8em}%
        \setlength{\labelsep}{0.5em}%
        \setlength{\itemsep}{1pt}%
        \setlength{\parsep}{0pt}%
        \setlength{\topsep}{1pt}%
        \setlength{\partopsep}{0pt}}}%
    {\end{list}}

\definecolor{ProjectLink}{HTML}{234F9F}
\newcommand{\projectresource}[3]{%
    \href{#1}{\mbox{%
        \makebox[1.45em][c]{\raisebox{-0.2\height}{%
            \includegraphics[height=1.25em]{#2}}}%
        \hspace{0.35em}\textcolor{ProjectLink}{\fontfamily{zi4}\selectfont #3}%
    }}%
}

\AtBeginDocument{%
  \setlength{\abovedisplayskip}{6pt plus 2pt minus 2pt}%
  \setlength{\belowdisplayskip}{6pt plus 2pt minus 2pt}%
  \setlength{\abovedisplayshortskip}{2pt plus 1pt}%
  \setlength{\belowdisplayshortskip}{3pt plus 1pt minus 1pt}}

\title{Agentic Knowledge-Model Co-Evolution via Consistency Regularization}

\author{Sirui Chen\textsuperscript{1,*}, Zhining Liu\textsuperscript{1}, Chao Zhang\textsuperscript{2}, Jingrui He\textsuperscript{1}, Zheng Li\textsuperscript{3,\dag} \\
\textsuperscript{1}University of Illinois Urbana-Champaign \quad \textsuperscript{2}Georgia Institute of Technology \quad \textsuperscript{3}Amazon \\
\texttt{sirui6@illinois.edu}
}

\iclrfinalcopy
\begin{document}

\maketitle
\lhead{Preprint}
{\renewcommand{\thefootnote}{\fnsymbol{footnote}}%
 \footnotetext[1]{This work was done during an internship at Amazon.}%
 \footnotetext[2]{Corresponding author.}}
\par\vspace{-12pt}
{\centering\normalfont\sffamily\fontsize{9}{11}\selectfont
\projectresource{https://github.com/SPLICE-Evolution/SPLICE}{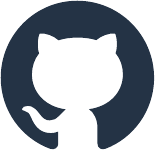}{Code}%
\hspace{1.5em}%
\projectresource{https://splice-evolution.github.io/splice.github.io/}{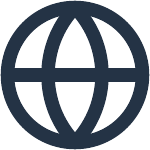}{Website}\par}
\vspace{16pt}
\etocdepthtag.toc{mtchapter}

\input{Sections/abstract}

\input{Sections/intro}
\input{Sections/prelim}
\input{Sections/method}
\input{Sections/exps}

\begingroup
\input{Sections/related_work}
\input{Sections/conclusion}
\endgroup
\newpage
\bibliography{ref}
\bibliographystyle{main}
\newpage
\appendix
\etocdepthtag.toc{mtappendix}
\etocsettagdepth{mtchapter}{none}
\etocsettagdepth{mtappendix}{subsection}
\renewcommand{\contentsname}{Appendix}
{\hypersetup{linkcolor=black}\tableofcontents}
\clearpage
\input{Sections/appendix}

\end{document}

%% file: math_commands.tex
\usepackage{amsmath,amsfonts,bm}

\def\eqref#1{equation~\ref{#1}}

\def\1{\bm{1}}

\DeclareMathAlphabet{\mathsfit}{\encodingdefault}{\sfdefault}{m}{sl}
\SetMathAlphabet{\mathsfit}{bold}{\encodingdefault}{\sfdefault}{bx}{n}



%% file: Sections/abstract.tex
\begin{abstract}
LLM agents can evolve by learning from accumulated experience. Existing approaches extract reusable external knowledge, such as skills and memories, from experience to guide subsequent task solving, with some further internalizing this knowledge through model training. 
However, different ways of analyzing and transforming the same experience can yield knowledge with substantially different utility, even when the resulting guidance appears reasonable. Inspired by consistency regularization, we identify consistent utility gains across different ways of analyzing the same experience, rather than relying on any single way of analyzing it.
We introduce \method{}, a framework for consistency-regularized knowledge-model co-evolution. It constructs evolution views by perturbing how the same experience is analyzed, producing different knowledge candidates. Knowledge updates are then guided by common gain, which accounts for utility gains and their consistency across views. During model evolution, self-distillation internalizes the evolved knowledge alongside reinforcement learning, using knowledge associated with consistent utility degradation as a contrastive reference. The two processes interleave, with model updates driving further knowledge evolution. Experiments across diverse domains show that \method{} outperforms knowledge evolution and model training baselines by up to 7.1 and 4.6 points in average accuracy, while its evolved knowledge reduces output tokens per correct {attempt} by up to 44\% and transfers across model sizes and families.
\end{abstract}

\begin{figure}[H]
    \centering
    \includegraphics[width=0.85\textwidth]{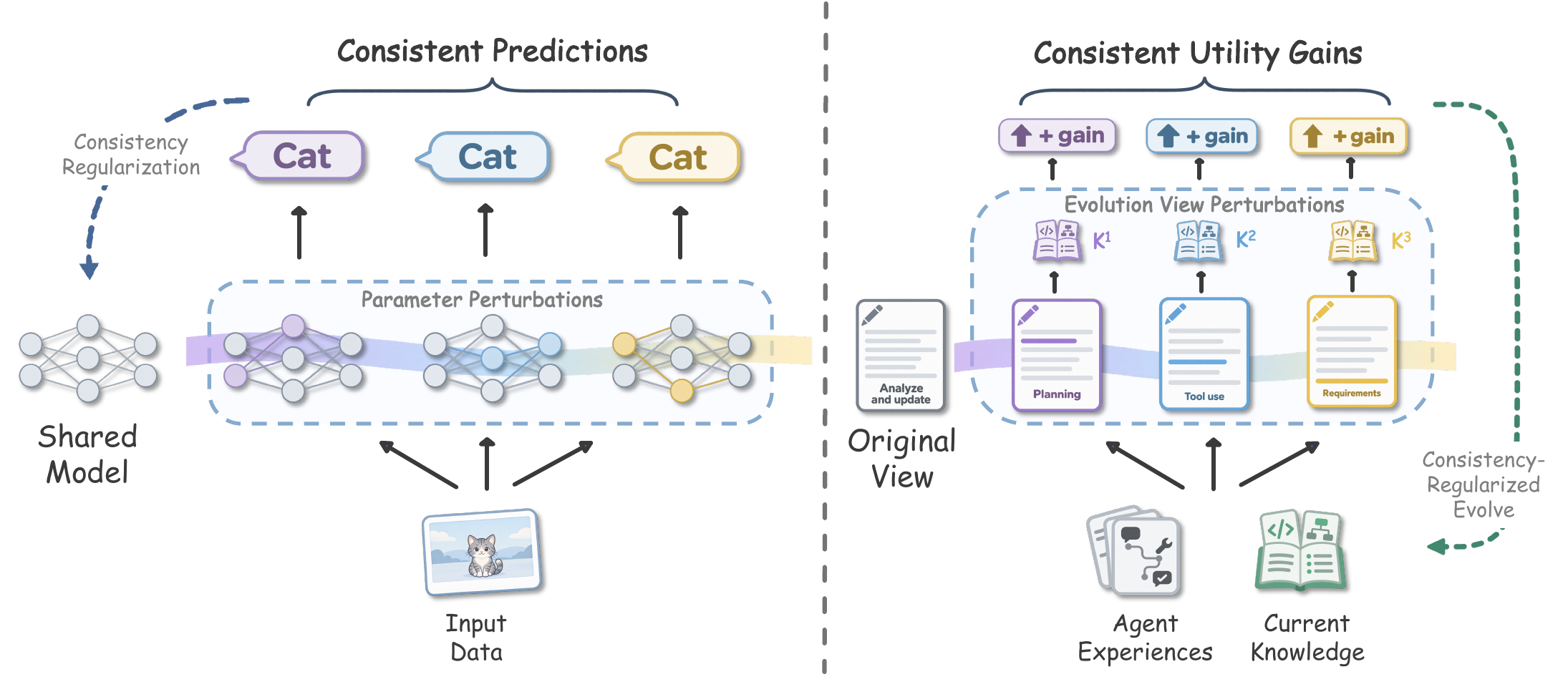}
    \caption{
        \textbf{From prediction consistency to utility consistency.}
        \textbf{Left:} Consistency regularization encourages consistent
        predictions under parameter perturbations.
        \textbf{Right:} \method{} guides knowledge updates using utility
        gains that persist across different evolution views of the same
        experience.
    }
    \label{fig:teaser}
    \vspace{-6pt}
\end{figure}

%% file: Sections/intro.tex
\section{Introduction}
\label{sec:intro}

Large Language Model (LLM) agents accumulate experience whenever they reason, use tools, and interact with the environment. For this experience to support continual improvement, an agent must identify useful guidance and reuse it in future tasks. This motivates knowledge evolution, through which agents extract and revise external knowledge, such as skills and memories, from their experience \citep{shinn2023reflexion,zhang2026agentic,ouyang2026reasoningbank,yang2026skillopt}. Such knowledge can also be internalized into model parameters through policy training \citep{lu2026self, lu2026skill0,he2026siri}. The updated model then generates new experience for further knowledge evolution, allowing knowledge and the model to co-evolve \citep{xia2026skillrl,he2026reskill}. External knowledge thus serves as a \emph{medium for agent evolution}, connecting accumulated experience with subsequent task solving and model learning.

However, knowledge derived from the same experience can differ substantially depending on which evidence is selected and how it is interpreted. For example, analyzing a purchase of the wrong size can yield guidance to include size requirements in queries or verify the selected size before purchase. On a product page that lists the requested size but defaults to another option, the first rule leaves the incorrect selection unchanged, while the second corrects it. Consequently, different ways of analyzing and transforming the same experience can yield knowledge with substantially different utility. Recent studies similarly find that even seemingly reasonable criteria for analyzing and transforming experience can yield less useful knowledge \citep{huang2026raw}, and that plausible knowledge can even underperform using no knowledge at all \citep{li2026skillsbench,dong2026agent}. This raises a central question: \emph{how can experience be effectively transformed into generalizable knowledge to support agentic evolution?}

To tackle this problem, we draw inspiration from consistency regularization (Figure~\ref{fig:teaser}). This approach encourages consistent predictions for the same input under random parameter perturbations~\citep{ni2024pace,wu2021r}. This reduces prediction sensitivity to such perturbations and can improve generalization. 
Similarly, we view the transformation of experience into knowledge as a function parameterized by textual instructions, whose perturbations define different evolution views. Here, consistency concerns utility: when different views produce distinct knowledge yet yield consistent positive gains on the same tasks, this suggests that the benefits are not specific to a single evolution view. Rather than relying on any single analysis, we therefore use utility consistency across evolution views to guide knowledge evolution toward more generalizable knowledge.

We introduce \method{}, a framework for consistency-regularized knowledge-model co-evolution. During knowledge evolution, it perturbs how the same experience is analyzed and transformed, creating \emph{evolution views} that each produce a knowledge candidate. Each candidate is evaluated through task replay, and a \emph{common gain} is calculated by combining utility gain with cross-view consistency. A knowledge optimizer then updates the current knowledge using the candidates and replay trajectories from tasks with the highest common gains. During model evolution, \method{} combines GRPO~\citep{shao2024deepseekmath} with self-distillation to internalize the evolved knowledge. The same replay also reveals guidance associated with consistent utility degradation, from which we construct noisy knowledge as a contrastive reference to highlight useful guidance in the evolved knowledge. A token-level gate controls its contribution to retain shared useful guidance. The two processes are interleaved, with the updated model generating new experience for further co-evolution.
We evaluate \method{} on four benchmarks covering math, coding, office work,
and web shopping at two model scales. Across these tasks, \method{}
consistently outperforms knowledge evolution and model training baselines,
improving average accuracy over the strongest baseline by 3.0 points for
Qwen3.5-4B and 4.6 points for Qwen3.5-9B. Notably, its evolved knowledge makes
task solving more efficient, requiring up to 44\% fewer output tokens per
correct {attempt} than the {most efficient baseline}, and delivers larger gains when transferred to other
model sizes and families.

%% file: Sections/prelim.tex
\section{Preliminaries}
\label{sec:prelim}
\paragraph{Knowledge-model co-evolution.}
We consider an LLM agent parameterized by $\theta$ and equipped with external knowledge $K$, including reusable skills, tools, memories, or environment information. Given a task $x\sim\mathcal{D}$, the agent generates a trajectory $\tau$ using $\pi_\theta(\cdot\mid x,\tau_{<t},K)$, where $\tau_{<t}$ contains preceding model outputs and environment feedback. The task, trajectory, and feedback, including the reward $R(x,\tau)$, form the experience $\mathcal{E}$. Knowledge evolution updates $K$ from experience $\mathcal{E}$, yielding $\widetilde K=\mathcal{G}(\mathcal{E},K;p)$, where $\mathcal{G}$ is a knowledge optimizer parameterized by textual instructions $p$. This transformation typically involves analyzing trajectories and feedback, proposing reusable guidance, and organizing and reviewing knowledge updates. The instructions $p$ specify which aspects of the experience are examined and how the resulting guidance is incorporated into $K$. 

Model evolution then internalizes knowledge-guided behaviors into $\theta$ through training. The updated model generates new experience and can change the utility of $K$, measured by the model's expected task reward with that knowledge, informing subsequent knowledge evolution. Co-evolution interleaves these two processes to maximize the expected return
of the final policy under the evolved knowledge,
$J_\theta=\mathbb{E}_{x\sim\mathcal{D},\,\tau\sim\pi_\theta(\cdot\mid x,K)}[R(x,\tau)]$.
\paragraph{Group Relative Policy Optimization.}
\label{sec:grpo}
Group Relative Policy Optimization (GRPO)
\citep{shao2024deepseekmath}
optimizes the policy using group-relative advantages. For each task $x$, it samples $G$ trajectories $\{\tau_g\}_{g=1}^{G}$ from the behavior policy $\pi_{\text{old}}(\cdot\mid x)$ and minimizes
\begin{equation}
\begin{aligned}
    \mathcal{L}_{\mathrm{GRPO}}(\theta)
    ={}&
    -\mathbb{E}\left[
        \frac{1}{G}\sum_{g=1}^{G}
        \frac{1}{|\tau_g|}\sum_{t=1}^{|\tau_g|}
        \min\left(
            \rho_{g,t}A_g,\,
            \operatorname{clip}
            (\rho_{g,t},1-\varepsilon,1+\varepsilon)A_g
        \right)
    \right]
    +\beta\mathcal{D}_{\mathrm{KL}}(\theta),
\end{aligned}
\label{eq:prelim-grpo-loss}
\end{equation}
where the expectation is over tasks and sampled trajectory groups. Here, $\rho_{g,t}$ is the token probability ratio between $\pi_\theta$ and $\pi_{\text{old}}$, $A_g$ is the group-normalized advantage of trajectory $\tau_g$, $\varepsilon$ is the clipping threshold, and $\beta$ weights the KL penalty toward a frozen reference policy.

%% file: Sections/method.tex
\section{\method{}: Knowledge-Model Co-Evolution}
\label{sec:method}
We introduce \method{}, a framework for consistency-regularized
knowledge-model co-evolution, illustrated in
Figure~\ref{fig:overview}.
Knowledge evolution (\S~\ref{sec:knowledge-evolution}) uses
utility consistency across evolution views to guide updates
to external knowledge.
Model evolution (\S~\ref{sec:internalization}) internalizes
the evolved knowledge through GRPO and gated contrastive
self-distillation.
The two processes are interleaved, with the updated model
generating new experience for further knowledge evolution.
\begin{figure}[t]
    \centering
    \includegraphics[width=0.92\linewidth]{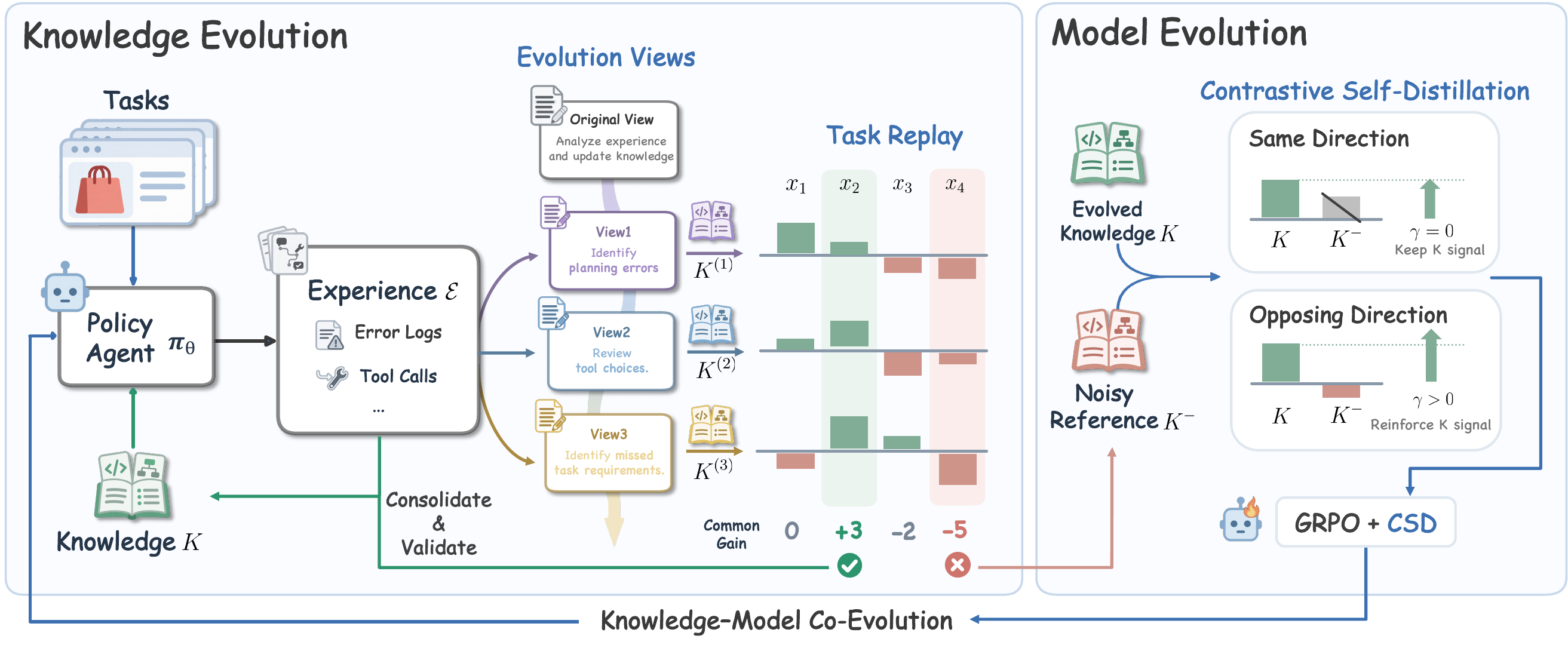}
    \caption{
        \textbf{Overview of \method{}.}
        \textbf{Left:} Knowledge evolution perturbs how the same
        experience is analyzed and transformed, using common gain
        across evolution views to select task experience for
        knowledge updates.
        \textbf{Right:} Model evolution combines GRPO with gated
        contrastive self-distillation to internalize the evolved
        knowledge, using noisy knowledge as a contrastive reference.
    }
    \label{fig:overview}
    \vspace{-4pt}
\end{figure}
\subsection{Consistency-Regularized Knowledge Evolution}
\label{sec:knowledge-evolution}

\paragraph{Knowledge consistency evaluation.}
At each round, the LLM agent $\pi_\theta$ solves a batch of tasks $\{x_i\}_{i=1}^{B}$ with the current knowledge $K$ to collect experience $\mathcal{E}$. We then ask the knowledge optimizer $\mathcal{G}$ to generate $N$ textual revisions of $p$. These revisions can emphasize different aspects of a failure, such as planning errors, ineffective tool use, or missed task requirements. Each revision defines a distinct \emph{evolution view}, which is applied to the same $\mathcal{E}$ and $K$ to produce a knowledge candidate:
\begin{equation}
    K^{(j)}
    =
    \mathcal{G}(\mathcal{E},K;p\oplus\delta_j),
    \qquad j=1,\ldots,N.
    \label{eq:knowledge-views}
\end{equation}
Here, $\delta_j$ denotes a textual revision, and $\oplus$ incorporates it into $p$.
To assess the utility of each knowledge candidate $K^{(j)}$, we replay the tasks used to collect $\mathcal{E}$ with that candidate, enabling direct utility comparisons against the current knowledge $K$.
For task $x_i$, we define $\Delta_{ij}$ as the mean reward gain under $K^{(j)}$ relative to the current knowledge $K$ across replay runs.
A high observed gain from a single candidate can reflect useful knowledge or a particularly successful execution. We therefore select task experience supported by consistent utility gains across evolution views, rather than choosing a candidate solely by its average gain across tasks. For each task $x_i$, we compute the \emph{common gain}, which combines the mean utility gain with an inconsistency penalty:
\begin{equation}
\begin{aligned}
    \widehat C_i
    &=
    \underbrace{\frac{1}{N}\sum_{j=1}^{N}\Delta_{ij}}_{\text{mean gain}}
    -
    \underbrace{
        \frac{1}{2\binom{N}{2}}
        \sum_{j<\ell}|\Delta_{ij}-\Delta_{i\ell}|
    }_{\text{inconsistency penalty}}
    =
    \frac{1}{\binom{N}{2}}
    \sum_{j<\ell}\min(\Delta_{ij},\Delta_{i\ell}).
    \label{eq:common-gain}
\end{aligned}
\end{equation}
The pairwise-minimum form measures the gain shared by two evolution views. Averaging over pairs favors tasks with gains consistently supported across views, while discounting isolated high gains that can reflect an occasional success or guidance specific to a particular view. We select up to $n$ tasks with the highest positive common gains ($\widehat C_i>0$). Their trajectories provide evidence of improvements supported across evolution views, forming the basis for the subsequent knowledge update. The following proposition further connects population common gain to the probability of obtaining positive utility gains under a new evolution view.
\begin{propositionbox}
\begin{restatable}[Improvement under a new evolution view]
    {proposition}{viewstability}
\label{prop:view-stability}
Fix a task $x_i$, experience $\mathcal{E}$, model $\theta$, and current knowledge $K$. Let $\Delta$ and $\Delta'$ denote the utility gains from two independent, identically distributed evolution views, measured by the expected reward gain relative to $K$. Assume $\mathbb{E}|\Delta|<\infty$, and define the population mean gain $\mu=\mathbb{E}[\Delta]$ and common gain $C=\mathbb{E}[\min(\Delta,\Delta')]$. If $C>0$, then $\mu>0$ and
\begin{equation}
    \Pr(\Delta>0)
    \geq \frac{C}{\mu}
    =
    1-\frac{\mathbb{E}|\Delta-\Delta'|}{2\mu}.
    \label{eq:view-stability-bound}
\end{equation}
\end{restatable}
\end{propositionbox}

The proof is deferred to Appendix~\ref{app:view-stability-proof}.
The bound depends on two quantities: the mean utility gain $\mu$ and
the utility inconsistency $\frac{1}{2}\mathbb{E}|\Delta-\Delta'|$.
The mean gain measures the average benefit of the generated knowledge,
while inconsistency measures how much this benefit changes across
evolution views. At the same positive mean gain, lower inconsistency gives a stronger lower bound on the probability
of improvement under a new view. We further discuss the properties of common-gain selection in detail in Appendix~\ref{app:additional-common-gain-analysis}.

\paragraph{Knowledge evolution from common gain.}
Let $\mathcal S$ denote the selected tasks, $\mathcal{E}_{\mathcal S}\subseteq\mathcal{E}$ their original experience under $K$, and $\mathcal{E}_{\mathcal S}^{(j)}$ their replay experience under $K^{(j)}$. The knowledge optimizer $\mathcal{G}$ compares these trajectories to relate changes in knowledge to changes in task behavior, and identifies updates supported by consistent utility gains across evolution views, rather than simply extracting content shared by the candidates. It consolidates these updates into $K$ to produce
\begin{equation}
    \widetilde K=\mathcal{G}\big(\mathcal{E}_{\mathcal S},\{K^{(j)},\mathcal{E}_{\mathcal S}^{(j)}\}_{j=1}^{N},K;p'\big),
    \label{eq:knowledge-update}
\end{equation}
where $p'$ denotes the instructions for knowledge update. We evaluate $\widetilde K$ against $K$ on a held-out validation set and accept it only if validation performance improves; otherwise, $K$ remains unchanged. Subsequent rounds repeat this process with new experience under the current model and knowledge.

\subsection{Knowledge-Guided Model Evolution}
\label{sec:internalization}

A natural way to internalize the evolved knowledge $K$ is self-distillation~\citep{zhao2026self}, using the model conditioned on $K$ as a teacher. Our knowledge evolution process can also reveal patterns associated with consistent utility degradation. Prior work shows that learning from negative samples can provide useful supervision by discouraging incorrect responses~\citep{zhu2026surprising}. Motivated by these findings, we construct noisy knowledge from these degradation patterns and use it as a contrastive reference to complement the guidance from $K$ during self-distillation.

\paragraph{Contrastive knowledge guidance.}
To construct this reference, we follow the task-selection procedure in \S~\ref{sec:knowledge-evolution}, replacing $\Delta_{ij}$ with $-\Delta_{ij}$ when computing common gain. This selects tasks with utility degradation consistently supported across evolution views. For the selected tasks, the knowledge optimizer $\mathcal{G}$ similarly extracts and consolidates guidance associated with utility degradation into noisy knowledge $K^-$ for contrastive self-distillation. We construct $K^-$ anew whenever model evolution starts training with $K$, rather than evolving $K^-$ across knowledge evolution rounds.
\paragraph{Gated knowledge internalization.}
Knowledge that improves overall task performance can fail to benefit some tasks, while knowledge associated with degradation can still contain useful guidance. We therefore weight self-distillation by the observed task-level benefit of $K$ and gate the contribution of $K^-$ at the token level.
At the start of each training batch, we collect trajectories with and without $K$ for each task $x_i$ using the current policy $\pi_\theta$. Let $\overline R_i^{K}$ and $\overline R_i^{\varnothing}$ denote the mean rewards of the two rollout groups. We define the task-level weight $w_i=\max(\overline R_i^{K}-\overline R_i^{\varnothing},0)$, so only tasks with positive observed gains contribute to contrastive self-distillation, with larger gains receiving greater weight.
At the token level, we gate the contribution of $K^-$ to preserve useful guidance shared with $K$. For each token $y_t$ in a trajectory sampled without knowledge, we compare the model's predictions under $K$, $K^-$ and no knowledge on the same history. Let
$\ell_t(K')=\log\pi_{\theta}(y_t\mid x_i,\tau_{<t},K')$
and
$\ell_t=\log\pi_{\theta}(y_t\mid x_i,\tau_{<t})$
denote its log-probabilities with and without knowledge.
We define the token-level gate $\gamma_t$ and training signal $a_t$ as
\begin{equation}
\begin{aligned}
\gamma_t
&=
\operatorname{clip}\!\left(
-[\ell_t(K)-\ell_t][\ell_t(K^-)-\ell_t],
0,1
\right), \\
a_t
&=
\operatorname{sg}\!\left[
\sigma\!\left(
\ell_t(K)-\ell_t
-\gamma_t[\ell_t(K^-)-\ell_t]
\right)
\right],
\end{aligned}
\label{eq:gated-guidance}
\end{equation}
where $\sigma$ is a bounded, sign-preserving activation
(e.g., $\tanh$), and $\operatorname{sg}$ stops gradients.
When $K$ and $K^-$ change the token probability in the same direction, $\gamma_t=0$ and supervision comes from $K$ alone. When their effects oppose each other, $\gamma_t$ controls how strongly $K^-$ reinforces the signal from $K$. The activation $\sigma$ limits the magnitude of this signal while preserving its sign.
Combining the task-level weight with the token-level signal gives
\begin{equation}
\begin{aligned}
\mathcal{L}_{\mathrm{CSD}}
&=
-\mathbb{E}\!\left[
w_i a_t
\log\pi_\theta(y_t\mid x_i,\tau_{<t})
\right], \\
\mathcal{L}_{\mathrm{SPLICE}}
&=
\mathcal{L}_{\mathrm{GRPO}}
+\lambda\mathcal{L}_{\mathrm{CSD}}.
\end{aligned}
\label{eq:joint-loss}
\end{equation}
Here, $\lambda$ balances the two losses: $\mathcal{L}_{\mathrm{GRPO}}$ optimizes task reward on rollouts with and without $K$, and $\mathcal{L}_{\mathrm{CSD}}$ internalizes $K$ weighted by its observed benefit.

\paragraph{Knowledge-model co-evolution in \method{}.}
Co-evolution jointly updates the agent's external knowledge $K$ and model parameters $\theta$. Knowledge evolution proposes updates to $K$ through Eq.~\ref{eq:knowledge-update}, and model evolution uses the adopted knowledge to update $\theta$ through Eq.~\ref{eq:joint-loss}. The updated model generates new trajectories and feedback that provide experience for subsequent knowledge evolution. Since knowledge utility depends on the task model, proposed knowledge is evaluated with respect to the current model as training proceeds. We detail the co-evolution procedures in Appendix~\ref{app:co-evolution}.

%% file: Sections/exps.tex
\section{Experiments}

\subsection{Experimental Setup}
\label{sec:experimental-setup}

\paragraph{Benchmarks and baselines.}
We evaluate \method{} on four benchmarks from diverse domains:
LiveMathBench (LiveMath)~\citep{he2026livemathematicianbench} for
mathematical reasoning, LiveCodeBench (LiveCode)~\citep{jain2025livecodebench}
for code generation, SpreadsheetBench
(SpreadSheet)~\citep{ma2024spreadsheetbench} for spreadsheet manipulation,
and WebShop~\citep{yao2022webshop} for interactive web shopping, and report
accuracy (Acc.) on the test splits. We compare \method{} with two groups of
baselines. Training-free knowledge evolution methods, including
ReasoningBank~\citep{ouyang2026reasoningbank}, ACE~\citep{zhang2026agentic},
EvoSkill~\citep{alzubi2026evoskill}, and SkillOpt~\citep{yang2026skillopt},
evolve external knowledge with the task model fixed. Training-based methods
include GRPO~\citep{shao2024deepseekmath}, which optimizes the policy with
reinforcement learning, and SkillRL~\citep{xia2026skillrl} and
ReSkill~\citep{he2026reskill}, which co-evolve skills with the policy. We
also report the base model without knowledge. Details of the benchmarks and
baselines are provided in Appendix~\ref{app:implementation-details}.

\paragraph{Models and implementation details.}
We use Qwen3.5-4B and Qwen3.5-9B~\citep{qwen3.5} as task models. We evaluate
two variants of \method{}: \method{} (Know. Evo), which evolves knowledge with
the task model fixed, and \method{} (Co-Evo), which co-evolves knowledge and
the model with the asynchronous procedure in
Appendix~\ref{app:asynchronous-co-evolution}. {Unless otherwise specified,
we use Claude Sonnet 4.6~\citep{anthropic2026sonnet46} for knowledge generation
and editing.} {In \method{}, each} knowledge evolution round uses three evolution views
{and one replay per task under each candidate}.
Further implementation details and hyperparameters are provided in
Appendix~\ref{app:implementation-details}.

\begin{table*}[t]
\centering
\caption{\textbf{Main results (Acc., \%; $\uparrow$).} For each model size,
the upper group contains training-free methods, whose results are averaged
over three seeds, and the lower group contains methods that involve model
training. Avg. is computed over all four benchmarks, and $\Delta$ shows
improvement in Avg. over Base. \textbf{Bold} and \underline{underline} mark
the best and second-best results for each model size.}
\label{tab:main-results}
\small
\setlength{\tabcolsep}{7pt}
\renewcommand{\arraystretch}{1.05}
\definecolor{mainOurs}{RGB}{241,246,251}
\definecolor{mainDelta}{HTML}{6B6B6B}
\newcommand{\mainna}{\textemdash{}}
\begin{tabularx}{\textwidth}{Xccccc>{\color{mainDelta}}c}
\toprule
\textbf{Method} & \textbf{LiveMath} & \textbf{LiveCode} & \textbf{SpreadSheet} & \textbf{WebShop} & \textbf{Avg.} & $\boldsymbol{\Delta}$ \\
\midrule
\rowcolor{black!5}
\multicolumn{7}{l}{\rule{0pt}{2.4ex}\textbf{Qwen3.5-4B}} \\
Base & 27.59 & 38.89 & 11.43 & 3.60 & 20.38 & \mainna \\
ReasoningBank & 37.97 & 34.31 & 12.26 & 18.93 & 25.87 & +5.5 \\
ACE & 48.01 & 31.81 & 16.31 & 27.67 & 30.95 & +10.6 \\
EvoSkill & 45.70 & 40.83 & 14.17 & 44.33 & 36.26 & +15.9 \\
SkillOpt & 47.46 & 39.31 & 17.38 & 42.60 & 36.69 & +16.3 \\
\rowcolor{mainOurs}
\textbf{SPLICE} (Know. Evo) & 50.33 & 47.08 & 17.38 & 58.20 & 43.25 & +22.9 \\
\midrule
GRPO & 51.99 & \underline{55.00} & 21.43 & 61.00 & 47.35 & +27.0 \\
SkillRL & 60.60 & 48.75 & 20.00 & 67.00 & 49.09 & +28.7 \\
ReSkill & \underline{61.59} & 52.50 & \textbf{28.93} & \underline{67.40} & \underline{52.60} & \underline{+32.2} \\
\rowcolor{mainOurs}
\textbf{SPLICE} (Co-Evo) & \textbf{67.55} & \textbf{59.58} & \underline{25.00} & \textbf{70.20} & \textbf{55.58} & \textbf{+35.2} \\
\midrule
\rowcolor{black!5}
\multicolumn{7}{l}{\rule{0pt}{2.4ex}\textbf{Qwen3.5-9B}} \\
Base & 31.02 & 50.14 & 22.26 & 9.07 & 28.12 & \mainna \\
ReasoningBank & 39.85 & 47.36 & 28.81 & 22.93 & 34.74 & +6.6 \\
ACE & 50.77 & 51.94 & 28.33 & 31.93 & 40.75 & +12.6 \\
EvoSkill & 48.23 & 46.25 & 29.64 & 48.27 & 43.10 & +15.0 \\
SkillOpt & 49.23 & 51.11 & 26.07 & 48.53 & 43.74 & +15.6 \\
\rowcolor{mainOurs}
\textbf{SPLICE} (Know. Evo) & 58.06 & 55.56 & 29.88 & 59.67 & 50.79 & +22.7 \\
\midrule
GRPO & 67.22 & \underline{62.50} & 27.86 & \underline{69.00} & 56.64 & +28.5 \\
SkillRL & 63.58 & 59.17 & 31.43 & 63.20 & 54.34 & +26.2 \\
ReSkill & \underline{70.20} & 62.08 & \underline{34.64} & 62.60 & \underline{57.38} & \underline{+29.3} \\
\rowcolor{mainOurs}
\textbf{SPLICE} (Co-Evo) & \textbf{71.52} & \textbf{64.17} & \textbf{38.21} & \textbf{74.00} & \textbf{61.98} & \textbf{+33.9} \\
\bottomrule
\end{tabularx}
\par\vspace{-4pt}
\begin{minipage}{\textwidth}
\footnotesize
\end{minipage}

\end{table*}

\subsection{Main Results}
\label{sec:main-results}

\paragraph{Knowledge evolution.}
With the task model fixed, \method{} (Know. Evo) matches or outperforms all
knowledge evolution baselines on every benchmark. It improves Avg. over the
strongest baseline, SkillOpt, by 6.6 and 7.1 points with Qwen3.5-4B and
Qwen3.5-9B, with the largest gains on WebShop (13.9 and 11.1 points), where
tasks require long-horizon interaction with the environment.

\paragraph{Knowledge-model co-evolution.}
\method{} (Co-Evo) achieves the highest Avg. among methods that update model
parameters, surpassing the strongest baseline, ReSkill, by 3.0 and 4.6 points
with Qwen3.5-4B and Qwen3.5-9B. Notably, skill-augmented baselines do not
always improve over GRPO, suggesting that evolved skills with inconsistent
utility can interfere with policy learning on complex tasks. \method{}
mitigates this issue by internalizing only knowledge selected by utility
consistency across evolution views. Together, \method{} couples
consistency-regularized knowledge evolution with model evolution, allowing the
agent to benefit from useful external guidance while continually improving
its own policy.

\input{Sections/analysis}

%% file: Sections/analysis.tex
\begingroup
\raggedbottom
\FloatBarrier

\subsection{Analysis}
\label{sec:analysis}

\input{Tables/analysis-scoring}

\paragraph{Effect of consistency regularization.}
The central claim of \method{} is that utility gains consistent across
evolution views provide stronger evidence for updating knowledge than gains
observed under a single view. We examine this claim by tracking the
proposed updates over 50 evolution rounds on WebShop, comparing \method{}
(Know. Evo) with the knowledge evolution baselines. As shown in
Figure~\ref{fig:consistency-val-risk} (left), \method{} reaches 53\% and
49\% validation Acc. by round 10 with Qwen3.5-4B and Qwen3.5-9B and remains
stable afterwards, whereas the baselines fluctuate at lower levels.
Figure~\ref{fig:consistency-val-risk} (right) reports the fraction of
rounds in which the proposed update lowers validation Acc. by more than a
given margin. Only 6\% and 2\% of the updates proposed by \method{} cause
drops larger than 5 points with Qwen3.5-4B and Qwen3.5-9B, compared with up
to 46\% and 56\% for the baselines. Consistency across evolution views thus makes
updates less likely to degrade the current knowledge, in line
with Proposition~\ref{prop:view-stability}.

\wrapstuffclear

\begin{figure}[t]
    \centering
    \vspace{-6pt}
    \includegraphics[width=\linewidth]{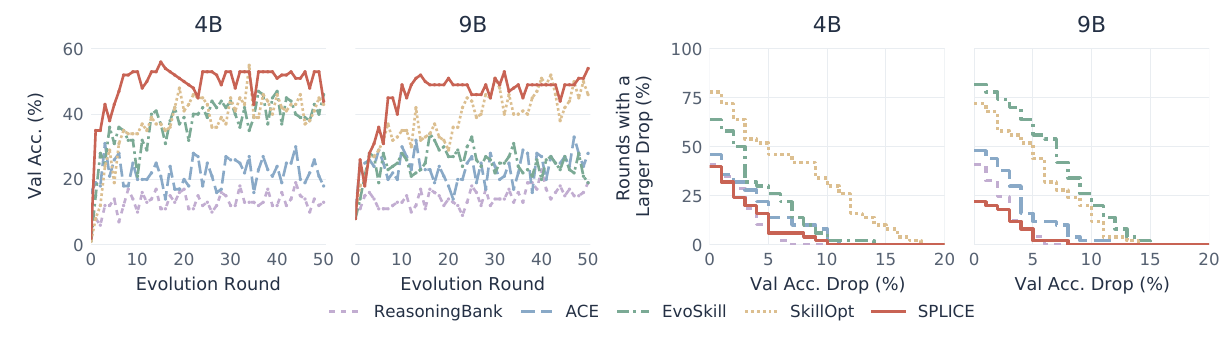}
    \caption{\textbf{Validation performance and degradation risk on WebShop.}
        Left: validation Acc. of the proposed update across evolution
        rounds, before acceptance. Right: percentage of rounds in which
        the proposed update causes an absolute validation Acc. drop
        larger than a given threshold (in points).}
    \label{fig:consistency-val-risk}
    \vspace{-4pt}
\end{figure}

\begin{wrapstuff}[r,type=figure,width=0.5\textwidth,top=0,abovesep=0pt,belowsep=0pt,leftsep=\columnsep]
    \centering
    \captionsetup{font=small,skip=4pt}
    \includegraphics[width=\linewidth]{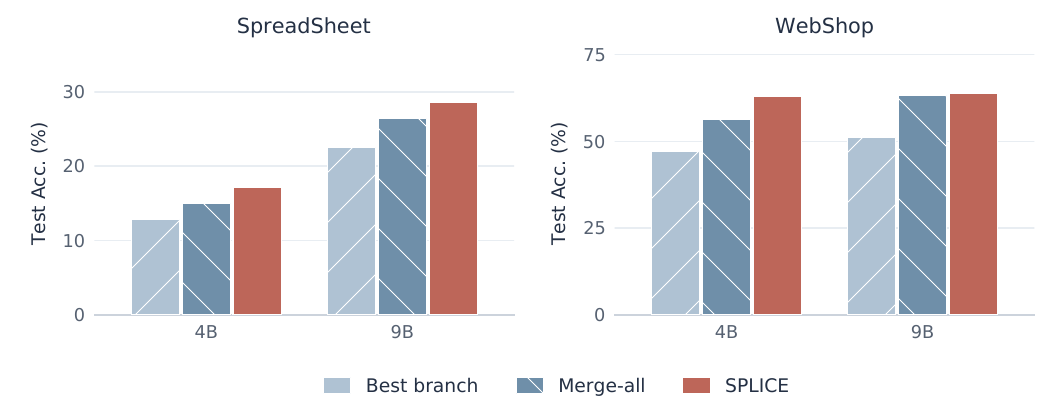}
    \caption{\textbf{Different strategies of using multiple evolution views on
        SpreadSheet and WebShop.} Best branch adopts the candidate with the
    highest mean replay reward, and Merge-all updates knowledge with the
    evidence from all tasks. \method{} (Know. Evo) achieves the best performance in all settings.}
    \label{fig:consolidation-comparison}
\end{wrapstuff}

As an ablation, Table~\ref{tab:scoring-control} replaces common gain with
mean gain when selecting task evidence. With mean gain, Test Acc. drops by
13.8 and 6.1 points with Qwen3.5-4B and Qwen3.5-9B, and validation Acc. is
maintained or improved in only 29\% of rounds. The consistency term is thus essential to
the effectiveness of \method{}. Appendix~\ref{app:additional-results} further analyzes how consistency
affects the stability of individual updates and their effects at the task
level.

\paragraph{Beyond scaling computation.}
\method{} spends additional computation on generating multiple evolution
views and replaying their candidates. To examine whether its gains come
merely from this additional computation, we compare it with two variants
that also generate multiple evolution views but use them differently: Best
branch adopts the candidate with the highest mean replay reward, and
Merge-all updates knowledge with the evidence from all tasks without
selection by common gain. As shown in
Figure~\ref{fig:consolidation-comparison}, \method{} achieves the highest
Test Acc. for both model sizes on SpreadSheet and WebShop. Best branch
performs worst, suggesting that relying on a single view discards useful
evidence, while the evidence merged by Merge-all can be redundant or
distracting. The gains of \method{} thus stem from how common gain uses the
evolution views, not from the additional computation alone.
\wrapstuffclear

\setlength{\textfloatsep}{10pt plus 2pt minus 2pt}
\input{Tables/analysis-efficiency}

\paragraph{Efficiency.}
Although \method{} draws on multiple knowledge candidates in each round,
we expect its knowledge to be compact but effective, because it retains
only updates supported by consistent utility gains across evolution views.
To examine whether this holds, we measure the final knowledge length of
\method{} (Know. Evo) and the knowledge evolution baselines, together with
the average output tokens and interaction steps over all WebShop test attempts,
including failures. As shown in Figure~\ref{fig:knowledge-efficiency}, \method{}
produces knowledge of 4.98k and 6.07k tokens for Qwen3.5-4B and Qwen3.5-9B,
respectively, compared with ACE's 41.51k and 64.99k. Its knowledge is also
shorter than EvoSkill's but longer than SkillOpt's. \method{} requires the
fewest interaction steps (8.33 and 8.64), suggesting that task-solving efficiency depends
on the retained guidance as well as knowledge length. ACE produces shorter outputs with
Qwen3.5-4B, but with much longer knowledge and more interaction steps. Table~\ref{tab:tokens-per-correct} further
reports the output tokens per correct attempt. \method{} has the lowest value on both
WebShop and LiveCode for both model sizes, reducing it by up to 44\%
relative to the {most efficient baseline}, indicating that it improves
accuracy {with fewer output tokens per correct attempt}.
\wrapstuffclear

\begin{figure}[t]
    \centering
    \captionsetup{skip=4pt}
    \includegraphics[width=0.9\linewidth]{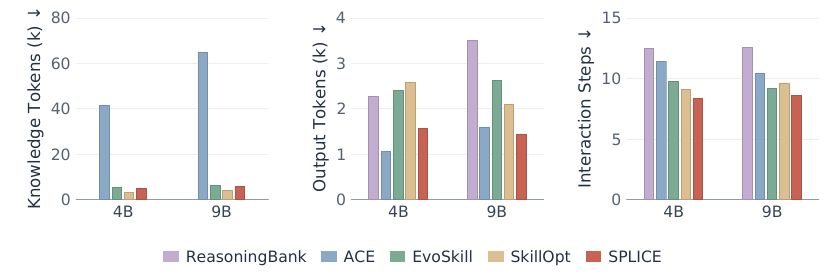}
    \caption{\textbf{Knowledge length, output tokens, and interaction steps on WebShop.}
    Results of \method{} (Know. Evo) and knowledge evolution baselines with
    Qwen3.5-4B and Qwen3.5-9B. ReasoningBank uses a dynamic knowledge bank, so its knowledge length is omitted.}
    \label{fig:knowledge-efficiency}
    \vspace{-4pt}
\end{figure}

\begin{figure}[t]
    \centering
    \captionsetup{skip=4pt}
    \includegraphics[width=0.9\linewidth]{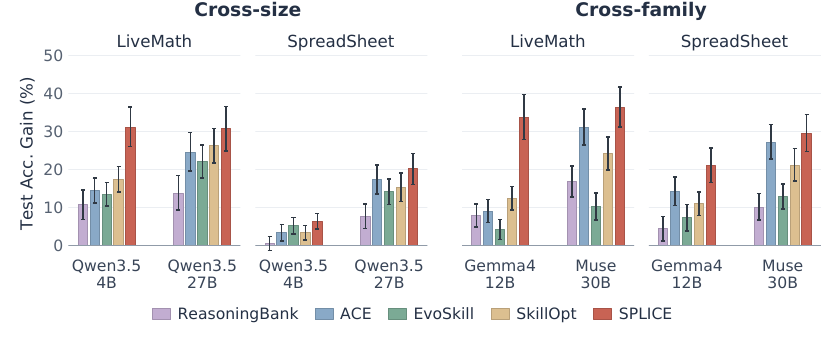}
    \caption{\textbf{Generalization across model sizes and families.}
    Knowledge evolved with Qwen3.5-9B is applied to other task models
    without further evolution. Bars show absolute gains in avg@8 Test Acc.
    (in points) over the same model without knowledge; error bars show paired 95\%
    confidence intervals.}
    \label{fig:generalization}
    \vspace{-4pt}
\end{figure}

\paragraph{Knowledge generalization.}
Just as consistency regularization improves model generalization, we
expect utility consistency to yield knowledge that generalizes across
task models. To examine whether
this holds, we take the final
knowledge evolved with Qwen3.5-9B by \method{} (Know. Evo) and the
knowledge evolution baselines, and apply it without further evolution to
task models of other sizes (Qwen3.5-4B and Qwen3.5-27B) and other
families (Gemma-4-12B-IT~\citep{gemmateam2026gemma4} and Muse-Glimmer-30B~\citep{museglimmer}). We measure absolute Test Acc.
gains (in points) over each target model without knowledge on LiveMath and
SpreadSheet. As shown in
Figure~\ref{fig:generalization}, \method{} yields the largest mean gain in all settings. On LiveMath, its advantage
over the baselines is particularly large for Qwen3.5-4B and Gemma-4-12B-IT:
the gains are 31.13 and 33.86 points, respectively, compared with at most
17.38 and 12.42 points for the baselines. Notably, the gains of \method{} on LiveMath remain at a similar level
across all four target models, whereas those of the baselines drop
substantially on Qwen3.5-4B and Gemma-4-12B-IT, indicating that the
knowledge evolved by \method{} is less sensitive to which model uses it.

\begin{figure}[t]
    \centering
    \begin{minipage}[c]{0.65\linewidth}
        \centering
        \includegraphics[width=\linewidth]{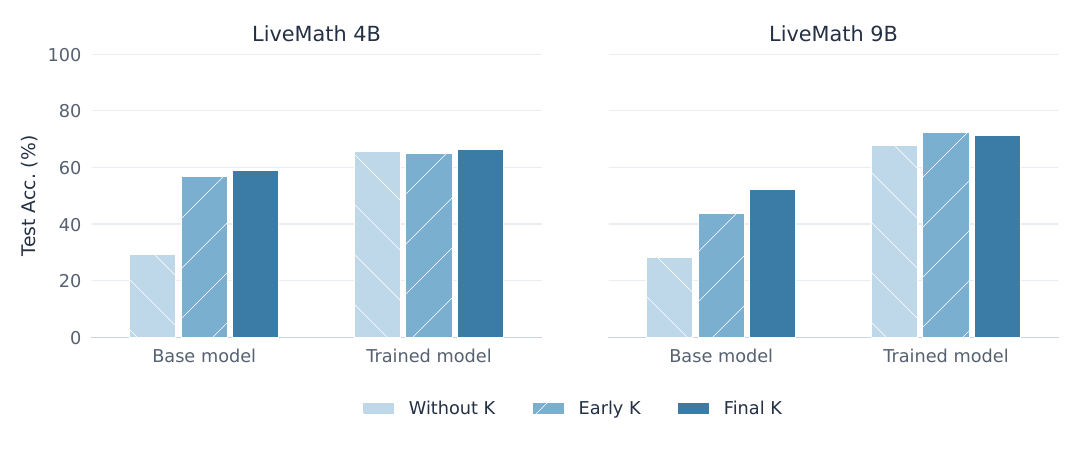}
    \end{minipage}\hfill
    \begin{minipage}[c]{0.33\linewidth}
        \captionsetup{font=small,skip=0pt}
        \caption{\textbf{Model and knowledge contributions on LiveMath.}
            Test Acc. of base and trained models evaluated without knowledge,
            with the knowledge at the start of training (Early $K$), and with
            the final knowledge (Final $K$).}
        \label{fig:training-model-knowledge}
    \end{minipage}
    \par\vspace{-8pt} 
\end{figure}

\input{Tables/analysis-optimizers}

\input{Sections/analysis_additional}

\paragraph{Co-evolution dynamics.}
Co-evolution trains the model with the evolved knowledge, which raises two
questions: whether the model internalizes the knowledge, and how the utility
of knowledge changes as the model evolves.
Figure~\ref{fig:training-model-knowledge} evaluates base and trained models on
LiveMath with knowledge from different stages. For the first question, the
trained models without knowledge already outperform the base models equipped
with the final knowledge for both model sizes, reaching 67.88\% versus
52.32\% with Qwen3.5-9B, indicating that model evolution internalizes and extends the evolved knowledge. For the second
question, knowledge brings much smaller gains to the trained models than to
the base models, and the early knowledge no longer improves the trained 4B
model. The training dynamics in Appendix~\ref{app:additional-results} show
the same trend during training: the benefit of the current knowledge shrinks
as the model improves without it. Validating knowledge with the current model
before using it for training therefore keeps model evolution from
internalizing knowledge that no longer helps.

\FloatBarrier
\endgroup

%% file: Tables/analysis-scoring.tex
\begin{wrapstuff}[r,type=table,width=0.52\textwidth,top=0,abovesep=4pt,belowsep=4pt,leftsep=\columnsep]
\vspace{-5pt}
\centering
\captionsetup{font=small,skip=4pt}
\caption{\textbf{Common gain versus mean gain on WebShop.} Mean gain
scores task evidence by the average utility gain across evolution views.
Gain$\geq$0 is the percentage of rounds in which
validation Acc. is not degraded.}
\label{tab:scoring-control}
\small
\setlength{\tabcolsep}{2.5pt}
\renewcommand{\arraystretch}{1.05}
\begin{tabular}{@{}llcc@{}}
\toprule
\textbf{Model} & \textbf{Metric} & \textbf{Mean gain} & \textbf{Common gain} \\
\midrule
Qwen3.5-4B & Test Acc. (\%) & 44.40 & \textbf{58.20} \\
           & Gain$\geq$0 (\%) & 28.57 & \textbf{67.86} \\
\midrule
Qwen3.5-9B & Test Acc. (\%) & 53.60 & \textbf{59.67} \\
           & Gain$\geq$0 (\%) & 28.57 & \textbf{69.05} \\
\bottomrule
\end{tabular}
\end{wrapstuff}

%% file: Tables/analysis-efficiency.tex
\begin{wrapstuff}[r,type=table,width=0.5\textwidth,top=0,abovesep=0pt,belowsep=4pt,leftsep=\columnsep]
\centering
\captionsetup{font=small,skip=4pt}
\caption{\textbf{Output tokens per correct attempt} (k,~$\downarrow$) on
test tasks, i.e., total output tokens divided by the number of correct
attempts. Bold marks the lowest.}
\label{tab:tokens-per-correct}
\small
\setlength{\tabcolsep}{3.5pt}
\renewcommand{\arraystretch}{1.05}
\begin{tabular}{@{}lcccc@{}}
\toprule
 & \multicolumn{2}{c}{\textbf{WebShop}} & \multicolumn{2}{c}{\textbf{LiveCode}} \\
\cmidrule(lr){2-3}\cmidrule(l){4-5}
\textbf{Method} & 4B & 9B & 4B & 9B \\
\midrule
ReasoningBank & 11.99 & 15.25 & 24.86 & 14.85 \\
ACE & 3.81 & 4.99 & 27.47 & 15.27 \\
EvoSkill & 5.42 & 5.46 & 22.17 & 15.07 \\
SkillOpt & 6.04 & 4.31 & 25.17 & 14.79 \\
\method{} (Know. Evo)& \textbf{2.69} & \textbf{2.40} & \textbf{19.82} & \textbf{13.64} \\
\bottomrule
\end{tabular}
\end{wrapstuff}

%% file: Tables/analysis-optimizers.tex
\begin{table}[!t]
    \color{black}
    \arrayrulecolor{black}
    \centering
    \captionsetup{font=small,skip=4pt,labelfont={color=black},textfont={color=black}}
    \caption{\textbf{Alternative knowledge optimizers on WebShop.}
    Final Test Acc. (\%) of one complete knowledge-evolution run per entry.
    Bold marks the best method in each task-model/optimizer comparison.}
    \label{tab:optimizer-comparison}
    \small
    \setlength{\tabcolsep}{6pt}
    \renewcommand{\arraystretch}{1.05}
    \begin{tabular}{@{}llccc@{}}
        \toprule
        \textbf{Task model} & \textbf{Method} & \textbf{Haiku 4.5} & \textbf{Qwen3-Next} & \textbf{DeepSeek-V3.1} \\
        \midrule
        Qwen3.5-4B & SkillOpt & 12.40 & 2.80 & 2.00 \\
         & EvoSkill & 1.40 & 5.00 & 2.20 \\
         & \method{} (Know. Evo) & \textbf{49.80} & \textbf{39.40} & \textbf{21.40} \\
        \midrule
        Qwen3.5-9B & SkillOpt & 7.60 & 12.80 & 11.80 \\
         & EvoSkill & 7.80 & 8.60 & 11.20 \\
         & \method{} (Know. Evo) & \textbf{49.60} & \textbf{33.00} & \textbf{28.00} \\
        \bottomrule
    \end{tabular}
    \arrayrulecolor{black}
\end{table}

%% file: Sections/analysis_additional.tex
\input{Tables/analysis-replay-views}
\begingroup
\paragraph{Numbers of replays and evolution views.}
We examine how the numbers of replays and evolution views affect performance.
Table~\ref{tab:replay-view-ablation} varies the number of replays per task
under each candidate, $R$, and the number of evolution views, $N$. For each
task, we average the rewards from multiple replays under each candidate to
estimate its gain relative to the current knowledge, then compute common gain
across views. Increasing $R$ generally improves accuracy, but the gains are
modest. Increasing $N$ from two to three improves accuracy overall, whereas
four views do not yield consistent further gains. Since replay and candidate-generation
costs grow linearly with $R$ or $N$, further increases would not be
cost-effective.
Therefore, the default $R=1$ and $N=3$ offer a practical balance between performance and cost.

\paragraph{Across knowledge optimizers.}
We test whether \method{}'s advantage is robust across knowledge optimizers.
We evaluate Claude Haiku 4.5~\citep{anthropic2025haiku45}, Qwen3-Next-80B-A3B~\citep{qwen3technicalreport},
and DeepSeek-V3.1~\citep{liu2024deepseek} as alternative knowledge optimizers on WebShop
(Table~\ref{tab:optimizer-comparison}).
\method{} clearly outperforms the other two methods under every optimizer and
for both model sizes, whereas the two baselines nearly fail. One observed failure
mode is that weaker optimizers can produce lengthy, overly procedural knowledge.
For example, a skill generated by EvoSkill with Haiku 4.5 requires multi-step
validation before every search. The resulting repeated checks cause many action
outputs to be truncated. \method{} also produces candidates of varying quality,
but common gain excludes task evidence with negative gains from consolidation.
This aligns with our motivation:
by retaining only evidence with consistent gain across views, \method{} can
still extract generalizable knowledge.
\par
\endgroup

%% file: Tables/analysis-replay-views.tex
\newsavebox{\replayviewtablebox}
\sbox{\replayviewtablebox}{%
    \color{black}%
    \arrayrulecolor{black}%
    \small%
    \setlength{\tabcolsep}{3.5pt}%
    \renewcommand{\arraystretch}{1.05}%
    \begin{tabular}{@{}lcccc@{}}
        \toprule
         & \multicolumn{2}{c}{\textbf{WebShop}} & \multicolumn{2}{c}{\textbf{LiveMath}} \\
        \cmidrule(lr){2-3}\cmidrule(l){4-5}
        \textbf{Setting} & 4B & 9B & 4B & 9B \\
        \midrule
        Default & 60.40 & 66.40 & 52.32 & 53.97 \\
        \midrule
        $R=2$ & 62.00 & 68.80 & 55.30 & 53.31 \\
        $R=3$ & 64.60 & 67.40 & 55.63 & 56.29 \\
        \midrule
        $N=2$ & 59.00 & 65.80 & 51.99 & 55.30 \\
        $N=4$ & 64.40 & 65.40 & 49.34 & 54.64 \\
        \bottomrule
    \end{tabular}%
}
\arrayrulecolor{black}
\setlength{\intextsep}{4pt}
\begin{wraptable}{R}{\wd\replayviewtablebox}
    \centering
    \captionsetup{font=small,skip=4pt,width=\wd\replayviewtablebox,justification=raggedright,singlelinecheck=false,labelfont={color=black},textfont={color=black}}
    \caption{\textbf{Numbers of replays and evolution views.}
    Test Acc. (\%) of complete knowledge evolution.
    The default uses $R=1$ replay and $N=3$ views.}
    \label{tab:replay-view-ablation}
    \usebox{\replayviewtablebox}
\end{wraptable}

%% file: Sections/related_work.tex
\section{Related Work}
\label{sec:related-work}

\paragraph{Knowledge evolution for LLM agents.}
LLM agents improve without parameter updates by turning experience into external knowledge, from textual reflections \citep{shinn2023reflexion,zhao2024expel} to structured memories, playbooks, and skills that are continually updated \citep{ouyang2026reasoningbank,zhang2026agentic,alzubi2026evoskill,yang2026skillopt,zhang2026memskill}. As such knowledge accumulates, recent studies find that it does not reliably improve task performance \citep{li2026skillsbench,huang2026raw,dong2026agent} and that seemingly correct experience can mislead future tasks \citep{xiong2026memory}. Some methods therefore retain only updates with positive estimated gain \citep{tang2026memory}, evaluating each candidate on its own. Rather than relying on any single analysis, \method{} uses utility consistency across evolution views to guide knowledge evolution.

\paragraph{Knowledge internalization and co-evolution.}
\leavevmode{Recent work internalizes external knowledge into model parameters
through reinforcement learning and self-distillation
\citep{lu2026self,lu2026skill0,he2026siri,lin2026skillc}, while co-evolution
methods update knowledge alongside the policy
\citep{xia2026skillrl,he2026reskill,wu2025evolver,shi2026skill1}.
Context distillation provides a mechanism for internalization by training a
model without context to match its predictions with context
\citep{snell2022learning}.
On-policy self-distillation uses the model conditioned on privileged context
as its own teacher \citep{zhao2026self}, with extensions that modulate
token-level guidance \citep{li2026demopsd} or introduce contrastive references
\citep{pan2026rlcsd}. \method{} constructs its contrastive reference from
knowledge that consistently degrades utility during knowledge evolution,
and applies a token-level gate to retain guidance shared with the evolved
knowledge.}

\paragraph{Consistency regularization.}
Consistency regularization improves generalization by encouraging consistent
predictions for the same input under perturbations. It was first widely
adopted in semi-supervised learning
\citep{bachman2014learning,sajjadi2016regularization,tarvainen2017mean,xie2020unsupervised}
and has since been extended to supervised training and fine-tuning of large
models \citep{wu2021r,ni2024pace}. {These methods assume that the correct
output is invariant to the perturbation and enforce consistency through a
training loss. In knowledge evolution, this assumption does not hold at the
level of content, since different evolution views are expected to produce
different knowledge. \method{} therefore requires consistency only in the
utility gains of the resulting knowledge and uses it as a criterion for
selecting evidence.} Additional related work is discussed in Appendix~\ref{app:more-related-work}.

%% file: Sections/conclusion.tex
\section{Conclusion}
\label{sec:conclusion}

We introduce \method{}, a framework for consistency-regularized knowledge-model co-evolution. \method{} updates knowledge based on utility gains that are consistent across evolution views and internalizes the evolved knowledge through reinforcement learning and gated contrastive self-distillation. Across four benchmarks and two model scales, \method{} outperforms knowledge evolution and model training baselines, while its evolved knowledge enables more efficient task solving and transfers across model sizes and families. Overall, our results suggest that utility gains consistent across different ways of analyzing experience, rather than those observed under any single analysis, provide an effective signal for turning experience into generalizable knowledge that, when internalized through model evolution, supports long-term agentic evolution.

%% file: Sections/appendix.tex
\section{Implementation Details}
\label{app:implementation-details}

\input{Sections/appendix_implementation}
\input{Sections/appendix_prompts}

\clearpage
\section{Knowledge-Model Co-Evolution}
\label{app:co-evolution}

\input{Sections/appendix_coevolution}

\clearpage
\section{More Related Work}
\label{app:more-related-work}

\input{Sections/appendix_related}

\clearpage
\section{Additional Results}
\label{app:additional-results}

\input{Sections/appendix_results}
\clearpage
\section{Proof of Proposition~\ref{prop:view-stability}}
\label{app:view-stability-proof}

\begin{propositionbox}
\viewstability*
\end{propositionbox}
\begin{proof}
All expectations and probabilities below are over evolution views,
with the task, experience, model, and current knowledge fixed.
Recall that
\[
    \mu=\mathbb{E}[\Delta],
    \qquad
    C=\mathbb{E}[\min(\Delta,\Delta')].
\]
The assumption $\mathbb{E}|\Delta|<\infty$ and identical distributions
ensure that $C$ is finite, since
$|\min(\Delta,\Delta')|\leq|\Delta|+|\Delta'|$.
Furthermore, $\min(\Delta,\Delta')\leq\Delta$ implies $C\leq\mu$.
Thus, $C>0$ ensures $\mu>0$.

We establish a more general bound for any threshold $0\leq a<C$.
The following pointwise inequality holds:
\begin{equation}
    \min(\Delta,\Delta')-a
    \leq
    (\Delta-a)\mathbf{1}\{\Delta'>a\}.
    \label{eq:app-pointwise-bound}
\end{equation}
If $\Delta'\leq a$, the left-hand side is at most
$\Delta'-a\leq0$, while the right-hand side is zero.
If $\Delta'>a$, the right-hand side equals $\Delta-a$,
which is at least $\min(\Delta,\Delta')-a$.

Taking expectations and using independence gives
\begin{align}
    C-a
    &\leq
    \mathbb{E}\!\left[
        (\Delta-a)\mathbf{1}\{\Delta'>a\}
    \right] \\
    &=
    \mathbb{E}[\Delta-a]\Pr(\Delta'>a) \\
    &=
    (\mu-a)\Pr(\Delta>a),
\end{align}
where the last equality follows from the identical distributions
of $\Delta$ and $\Delta'$.
Since $\mu\geq C>a$, we obtain
\begin{equation}
    \Pr(\Delta>a)
    \geq
    \frac{C-a}{\mu-a},
    \qquad
    \Pr(\Delta\leq a)
    \leq
    \frac{\mu-C}{\mu-a}.
    \label{eq:app-threshold-bound}
\end{equation}
Setting $a=0$ yields
\[
    \Pr(\Delta>0)\geq\frac{C}{\mu}.
\]

Finally, using $\min(u,v)=(u+v-|u-v|)/2$, we have
\[
    C
    =
    \frac{
        \mathbb{E}[\Delta]+\mathbb{E}[\Delta']
        -\mathbb{E}|\Delta-\Delta'|
    }{2}
    =
    \mu-\frac12\mathbb{E}|\Delta-\Delta'|.
\]
Therefore,
\[
    \Pr(\Delta>0)
    \geq
    \frac{C}{\mu}
    =
    1-\frac{\mathbb{E}|\Delta-\Delta'|}{2\mu},
\]
which completes the proof.
\end{proof}

\newpage
\section{Discussion of Common-Gain Selection}
\label{app:additional-common-gain-analysis}
\subsection{Illustrative Comparison of Selection Strategies}
\label{app:structured-view-specific-gains}

We use a simplified model to examine how different
selection strategies respond to utility gains that are
specific to particular evolution views.
All strategies use the same knowledge optimizer and the
same budget: each selects the evidence of one task and
constructs one knowledge update from it.
This comparison isolates the selection of task evidence
from how the knowledge optimizer consolidates updates.

Consider two replayed tasks, $R$ and $S$.
The knowledge updates constructed from their evidence
have expected utility gains $g_R>g_S>0$ on held-out
tasks, relative to the current knowledge $K$.
Across $N\geq2$ independently sampled evolution views,
their replay gains satisfy
\begin{equation}
    \Delta_{Rj}=g_R,
    \qquad
    \Delta_{Sj}=g_S+aZ_j,
    \qquad
    Z_j\overset{\mathrm{iid}}{\sim}
    \operatorname{Bernoulli}(q),
    \quad j=1,\ldots,N,
    \label{eq:app-structured-source-gains}
\end{equation}
where $0<q<1$.
The term $aZ_j$ represents a view-specific gain that
appears on the replayed task under only some evolution
views.
We assume that this gain does not transfer to held-out
tasks.
Let $d=g_R-g_S$ and assume $a>d>0$, so that an
individual view can rank $S$ above $R$ despite its
lower held-out utility.
To isolate this effect, we treat the replay gains as
exact expected reward differences, without replay
estimation noise.

Single-view selection ranks the tasks by
$\Delta_{i1}$.
Mean-gain selection ranks them by
$\frac{1}{N}\sum_{j=1}^{N}\Delta_{ij}$.
Common-gain selection uses Eq.~\ref{eq:common-gain}.
After selection, the same knowledge optimizer constructs
the update from the evidence of the selected task.
For definiteness, ties are resolved in favor of $R$.

Let $H=\sum_{j=1}^{N}Z_j$ denote the number of views
under which $S$ receives the additional gain.
Then $H\sim\operatorname{Binomial}(N,q)$, and the
mean-gain and common-gain scores of $S$ are
\begin{equation}
\begin{aligned}
    s_S^{\mathrm{mean}}
    &=g_S+a\frac{H}{N},\\
    s_S^{\mathrm{common}}
    &=g_S+a\frac{H(H-1)}{N(N-1)}.
\end{aligned}
\label{eq:app-source-score-aggregation}
\end{equation}
The second expression follows because only pairs in
which both views receive the additional gain contribute
$a$ to the pairwise minimum.
Task $R$ has score $g_R$ under all selection strategies.

\begin{propositionbox}
\begin{restatable}[Selection under view-specific gains]
    {proposition}{structuredviewcomparison}
\label{prop:structured-view-comparison}
Under the model above, let $U_{\mathrm{single}}$,
$U_{\mathrm{mean}}$, and $U_{\mathrm{common}}$
denote the held-out utility of the selected update.
For every $N\geq2$,
\begin{equation}
    \mathbb{E}[U_{\mathrm{common}}]
    \geq
    \mathbb{E}[U_{\mathrm{mean}}].
    \label{eq:app-finite-selection-comparison}
\end{equation}
Furthermore, if
\begin{equation}
    aq^2<d<aq,
    \label{eq:app-selection-advantage-regime}
\end{equation}
then
\begin{equation}
\begin{aligned}
    \lim_{N\to\infty}
    \mathbb{E}[U_{\mathrm{common}}]
    &=g_R,\\
    \mathbb{E}[U_{\mathrm{single}}]
    &=g_R-dq,\\
    \lim_{N\to\infty}
    \mathbb{E}[U_{\mathrm{mean}}]
    &=g_S.
\end{aligned}
\label{eq:app-selection-utilities}
\end{equation}
Thus, common-gain selection has strictly higher
expected held-out utility than both alternatives
for sufficiently large $N$.
\end{restatable}
\end{propositionbox}

\begin{proof}
Since $a>d$, single-view selection chooses $S$
exactly when $Z_1=1$, with probability $q$.
For the other two strategies,
Eq.~\ref{eq:app-source-score-aggregation} gives
\begin{equation}
\begin{aligned}
    \Pr(\text{mean selects }S)
    &=
    \Pr\left(\frac{H}{N}>\frac{d}{a}\right),\\
    \Pr(\text{common selects }S)
    &=
    \Pr\left(
        \frac{H(H-1)}{N(N-1)}>\frac{d}{a}
    \right).
\end{aligned}
\label{eq:app-selection-error-probabilities}
\end{equation}
For every $H\in\{0,\ldots,N\}$,
\[
    \frac{H(H-1)}{N(N-1)}
    \leq
    \frac{H}{N}.
\]
Therefore, whenever common-gain selection chooses $S$,
mean-gain selection also chooses $S$.
Selecting $S$ instead of $R$ reduces held-out utility
by $d>0$, proving
Eq.~\ref{eq:app-finite-selection-comparison}.

By the law of large numbers,
\[
    \frac{H}{N}\longrightarrow q,
    \qquad
    \frac{H(H-1)}{N(N-1)}
    \longrightarrow q^2
\]
almost surely.
Under Eq.~\ref{eq:app-selection-advantage-regime},
the probability that mean-gain selection chooses $S$
converges to one, while the corresponding probability
for common-gain selection converges to zero.
The expected held-out utility of any selection rule is
$g_R-d\Pr(\text{select }S)$.
Substitution yields
Eq.~\ref{eq:app-selection-utilities}.
Because $0<q<1$ and $d>0$,
$g_R>g_R-dq>g_S$, establishing the strict comparison
for sufficiently large $N$.
\end{proof}

This example distinguishes collecting more views from
requiring utility consistency across them.
As $N$ increases, mean-gain selection more accurately
estimates $g_S+aq$, which still includes the
view-specific gain.
Common gain instead approaches $g_S+aq^2$, reducing
the contribution of gains supported by only a fraction
of the views.
When $aq^2<d<aq$, this difference changes which
task evidence is selected for the knowledge update.

The advantage depends on both the transfer assumption
and the relative size of the view-specific gain.
If $d>aq$, both multi-view strategies select $R$ with
probability approaching one.
If $d<aq^2$, both select $S$ with probability approaching
one.
The analysis therefore illustrates a regime in which
consistency helps identify more transferable updates;
it does not establish universal superiority over
mean-gain selection or merging all evidence.

\subsection{Robustness to Shifts in Evolution Views}
\label{app:view-distribution-shift}

Proposition~\ref{prop:view-stability} considers a new
evolution view drawn from the same distribution.
We now allow the distribution of views to change,
capturing shifts in how experience is analyzed and
transformed into knowledge.

Fix a task $x_i$, experience $\mathcal{E}$, model
$\theta$, and current knowledge $K$.
Let $P$ denote the original distribution of evolution
views, and let $\Delta(v)$ be the expected reward gain
of the knowledge candidate produced under view $v$.
Define
\[
    \mu=\mathbb{E}_{V\sim P}[\Delta(V)],
    \qquad
    C=\mathbb{E}_{V,V'\overset{\mathrm{iid}}{\sim}P}
    \left[\min(\Delta(V),\Delta(V'))\right].
\]

\begin{propositionbox}
\begin{restatable}[Utility under a shift in evolution views]
    {proposition}{viewshiftrobustness}
\label{prop:view-shift-robustness}
Assume
$\mathbb{E}_{V\sim P}|\Delta(V)|<\infty$.
Let $Q$ be another distribution over evolution views,
with $Q\ll P$, and define its density ratio
$w(v)=\frac{dQ}{dP}(v)$.
Suppose that, for some $\rho\geq0$,
\begin{equation}
    \operatorname{ess\,sup}_{P} w
    -
    \operatorname{ess\,inf}_{P} w
    \leq \rho.
\label{eq:app-view-shift-condition}
\end{equation}
Then
\begin{equation}
    \mathbb{E}_{V\sim Q}[\Delta(V)]
    \geq
    \mu-\rho(\mu-C)
    =
    (1-\rho)\mu+\rho C.
\label{eq:app-view-shift-bound}
\end{equation}
In particular, if $0\leq\rho\leq1$, then
\begin{equation}
    \mathbb{E}_{V\sim Q}[\Delta(V)]\geq C.
\label{eq:app-view-shift-common-bound}
\end{equation}
Thus, when $0\leq\rho\leq1$, a positive common gain
guarantees positive expected utility gain under $Q$.
\end{restatable}
\end{propositionbox}

\begin{proof}
Since $w$ is a density ratio,
$\mathbb{E}_{P}[w]=1$.
Condition~\ref{eq:app-view-shift-condition} therefore
implies $0\leq w\leq1+\rho$ almost surely under $P$.
Consequently, $\Delta$ is integrable under $Q$, and
\[
    \mathbb{E}_{Q}[\Delta]
    =
    \mathbb{E}_{P}[w\Delta]
    =
    \mu+\operatorname{Cov}_{P}(\Delta,w).
\]

Let $V,V'$ be independent draws from $P$, and write
$\Delta=\Delta(V)$, $\Delta'=\Delta(V')$,
$w=w(V)$, and $w'=w(V')$.
The covariance identity gives
\[
    \operatorname{Cov}_{P}(\Delta,w)
    =
    \frac12
    \mathbb{E}\left[
        (\Delta-\Delta')(w-w')
    \right].
\]
Because $|w-w'|\leq\rho$ almost surely,
\[
    \operatorname{Cov}_{P}(\Delta,w)
    \geq
    -\frac{\rho}{2}
    \mathbb{E}|\Delta-\Delta'|.
\]
Using
$C=\mu-\frac12\mathbb{E}|\Delta-\Delta'|$,
we obtain
\[
    \mathbb{E}_{Q}[\Delta]
    \geq
    \mu-\rho(\mu-C),
\]
which proves Eq.~\ref{eq:app-view-shift-bound}.
Finally, $C\leq\mu$, so for $0\leq\rho\leq1$,
\[
    (1-\rho)\mu+\rho C\geq C.
\]
\end{proof}

The parameter $\rho$ controls how strongly the relative
frequencies of evolution views can change.
For example, $w(v)\in[1/2,3/2]$ satisfies
Eq.~\ref{eq:app-view-shift-condition} with $\rho=1$.
The bound shows that the possible reduction in expected
utility is controlled by $\mu-C$, the utility inconsistency.
At a fixed mean gain, lower inconsistency therefore
provides a stronger guarantee under a change in the
view distribution.

This result concerns the utility of knowledge evolution
on a fixed task. It allows bounded reweighting of views
supported by $P$; it does not directly establish transfer
to arbitrary new tasks or preservation of utility after
the knowledge optimizer consolidates updates.

\begingroup
\subsection{Correlation among Evolution Views}
\label{app:view-independence}

Proposition~\ref{prop:view-stability} treats evolution views as independent
and identically distributed given the task, the experience, the model, and the
current knowledge. In practice, the views are generated by the same knowledge
optimizer from the same experience and may be correlated.

Independence enters only the probability bound. The decomposition
$C=\mu-\frac{1}{2}\mathbb{E}|\Delta-\Delta'|$ requires only identical
distributions, and the empirical common gain in Eq.~\ref{eq:common-gain} does
not rely on either assumption. Without independence, the argument in
Appendix~\ref{app:view-stability-proof} yields
\[
    \Pr(\Delta>0)\geq\frac{C-\operatorname{Cov}\big(\Delta,\mathbf{1}\{\Delta'>0\}\big)}{\mu},
\]
so the bound becomes weaker when the gains of different views are positively
associated.

Correlation also affects common gain itself. As the views become more
correlated, their gains differ less and the inconsistency penalty shrinks, so
common gain approaches mean gain, and the two coincide when all views are
identical. Correlation thus reduces how strongly the consistency criterion
filters evidence; in the extreme, common-gain selection reduces to mean-gain
selection.

Generating the views with the knowledge optimizer is one way to instantiate
\method{}. The views can also be predefined, for example to incorporate human
expertise in analyzing experience, and common gain applies without change.
\endgroup

\clearpage
\section{Case Studies}
\label{app:case-studies}

\input{Sections/appendix_cases}

%% file: Sections/appendix_implementation.tex
\subsection{Benchmarks}
\label{app:benchmarks}

Table~\ref{tab:appendix-data-splits} summarizes the data splits. Training tasks
provide experience for knowledge evolution and model training, validation
tasks are used to evaluate proposed knowledge updates, and test tasks are used
only for evaluation.

\begin{table}[ht]
    \centering
    \caption{\textbf{Data splits and training settings.} Counts refer to tasks,
    or shopping sessions for WebShop. Output tokens are the maximum number of
    tokens per model response.}
    \label{tab:appendix-data-splits}
    \small
    \setlength{\tabcolsep}{6pt}
    \begin{tabular}{lrrrrr}
        \toprule
        Benchmark & Train & Val. & Test & Output tokens & Epochs \\
        \midrule
        LiveMath & 86 & 43 & 302 & 16,384 & 3 \\
        LiveCode & 68 & 34 & 240 & 32,768 & 3 \\
        SpreadSheet & 80 & 40 & 280 & 4,096 & 3 \\
        WebShop & 1,000 & 100 & 500 & 512 & 1 \\
        \bottomrule
    \end{tabular}
\end{table}

\paragraph{LiveMath.}
We use LiveMathematicianBench~\citep{he2026livemathematicianbench} in its
multiple-choice setting without proof sketches, with problems from the
March--June 2026 subsets. The agent answers each problem by selecting one of
five options.

\paragraph{LiveCode.}
We use the code-generation subset of LiveCodeBench~\citep{jain2025livecodebench},
combining the v5 and v6 releases and splitting problems by contest so that no
contest is shared across splits. The agent writes a Python solution, which is
evaluated by the execution-based test runner of LiveCodeBench.

\paragraph{SpreadSheet.}
We use the Verified-400 subset of SpreadsheetBench~\citep{ma2024spreadsheetbench}.
Given an instruction and an input workbook, the agent writes Python code to
produce an output workbook and can revise the code using execution feedback.

\paragraph{WebShop.}
We use the official 1K-product environment with synthetic shopping
goals~\citep{yao2022webshop}. The agent navigates textual product pages through
search and click actions, including product-option selection and purchase.

\subsection{Models and Baselines}
\label{app:models-baselines}

\paragraph{Models.}
The task models are Qwen3.5-4B and Qwen3.5-9B~\citep{qwen3.5} with thinking
mode disabled, decoded with temperature 1.0, top-$p$ 0.95, and top-$k$ 20.
{All knowledge evolution methods, including \method{} and the baselines, use
Claude Sonnet 4.6 as the knowledge optimizer.
Appendix~\ref{app:optimizer-prompts} describes how \method{} calls the
optimizer and lists its prompts.}
Table~\ref{tab:appendix-hyperparameters} lists the hyperparameters of \method{}.

\begin{table}[ht]
    \centering
    \caption{\textbf{Hyperparameters of \method{}.}}
    \label{tab:appendix-hyperparameters}
    \small
    \setlength{\tabcolsep}{6pt}
    \begin{tabular}{llr}
        \toprule
        Component & Hyperparameter & Value \\
        \midrule
        Knowledge evolution & Evolution views $N$ & 3 \\
         & Replay runs per candidate & 1 \\
         & Selected tasks $n$ & {$\{1, 2, 3, 4\}$} \\
        \midrule
        Model evolution & Self-distillation weight $\lambda$ & {0.3} \\
         & Learning rate & $1\times10^{-6}$ \\
         & Training batch size & {16 (LiveMath), 8 (others)} \\
         & Group size & 16 (LiveMath), 8 (others) \\
        \midrule
        Co-evolution & Reward-gap thresholds $\epsilon_{\mathrm{sd}}=\epsilon_{\mathrm{end}}$ & 0.03 \\
         & Stage budgets $T_{\min}$, $T_{\max}$ & 3, 5 (WebShop: 10, 20) \\
         & Restart patience $L$ & 2 \\
         & Warmup rounds & 4 (WebShop, LiveMath), 8 (others) \\
        \bottomrule
    \end{tabular}
\end{table}

\paragraph{Baselines.}
Base uses the task model without knowledge.
ReasoningBank~\citep{ouyang2026reasoningbank} stores reasoning memories
distilled from experience and retrieves task-relevant entries.
ACE~\citep{zhang2026agentic} maintains an evolving playbook through reflection
and curation. EvoSkill~\citep{alzubi2026evoskill} builds skills from task
experience, and SkillOpt~\citep{yang2026skillopt} iteratively optimizes
external skills. GRPO~\citep{shao2024deepseekmath} trains the task model with
group-relative policy optimization, and SkillRL~\citep{xia2026skillrl} and
ReSkill~\citep{he2026reskill} combine skill evolution with policy
optimization. {All baselines primarily follow their original configurations reported in their papers.} All training-based methods are trained for the number of epochs
listed in Table~\ref{tab:appendix-data-splits}.

\subsection{Evaluation Protocol}
\label{app:evaluation-protocol}

We report accuracy (Acc.), the percentage of test tasks that are solved. A task
is solved if the selected option is correct on LiveMath, all test cases pass on
LiveCode, the output workbook matches the reference in all cases on
SpreadSheet, or the purchase receives the full reward on WebShop. Results of
training-free methods are averaged over three seeds. \method{} (Co-Evo)
evaluates the final model with the latest knowledge that passes held-out
validation with this model, and without knowledge if no knowledge passes.

%% file: Sections/appendix_prompts.tex
\begingroup
\subsection{Knowledge Optimizer and Prompts}
\label{app:optimizer-prompts}

In each knowledge evolution round, \method{} calls the knowledge optimizer for four purposes.
(i) A single call revises the initial analysis instructions (Prompt~\ref{prompt:propose}) into $N$ revisions, one for each evolution view (Prompt~\ref{prompt:views}).
(ii) Under each view, the optimizer proposes an update from the experience with its revised instructions, and then writes the proposal into a standalone candidate (Prompt~\ref{prompt:edit}).
(iii) From the evidence of the selected tasks, it proposes additions or corrections to the current knowledge and integrates them into the knowledge update (Prompts~\ref{prompt:cons-propose} and~\ref{prompt:cons-edit}).
(iv) In co-evolution, it summarizes the guidance changes of candidates with consistent utility degradation to construct $K^-$ (Prompt~\ref{prompt:negative}).
We list the prompts below and render their Markdown formatting for readability.
\endgroup

\clearpage
\begin{promptbox}[label={prompt:propose}]{Initial analysis instructions}
You analyze an agent's current knowledge and experience to propose knowledge that addresses its observed failures.

\promptsec{Your Task}

Examine the current knowledge document together with the experience. Identify how the existing guidance relates to the failures, and propose knowledge that could improve the agent's performance.

A subsequent optimizer call will turn your proposal into a new, short, standalone document for testing.

\promptsec{Inputs}

\begin{promptitems}
    \item The current knowledge document.
    \item Task descriptions, available tools, trajectories, and outcomes.
    \item The document length budget.
\end{promptitems}

\promptsec{Requirements}

\begin{promptitems}
    \item Examine where the agent struggled and how the current knowledge influenced, failed to guide, or was not followed in those decisions.
    \item Propose concrete guidance that addresses the observed failure mode. Explain what should change relative to the current guidance and why.
    \item Consider successful trajectories and counterexamples when assessing the proposed improvement.
    \item Keep the proposal focused, reusable, and compatible with the available tools. Include enough detail to apply the method.
    \item If no useful improvement is supported, return an empty proposal and explain why.
\end{promptitems}

\promptsec{Output}

Return one JSON object:

\begin{lstlisting}[style=promptjson]
{
  "proposal": "The guidance to test, including concrete methods.",
  "sources": ["IDs of the relevant supplied trajectories"],
  "rationale": "How the current knowledge relates to the observed failures and why the proposed guidance could help."
}
\end{lstlisting}
\end{promptbox}

\clearpage
\begin{promptbox}[label={prompt:views}]{Generating evolution views}
You design revisions of the analysis instructions for improving an agent's current knowledge.

\promptsec{Your Task}

Given the initial analysis instructions, generate one revision for each requested evolution view. Each revision will guide the knowledge optimizer to examine the same current knowledge and experience, then propose useful additions or corrections to address the observed failures.

\promptsec{Inputs}

\begin{promptitems}
    \item \texttt{initial\_strategy}: the initial analysis instructions.
    \item \texttt{branch\_ids}: the evolution views that need a revision.
\end{promptitems}

\promptsec{Requirements}

\begin{promptitems}
    \item Preserve the focus on improving the current knowledge by understanding its relationship to the observed failures.
    \item Vary how the knowledge optimizer analyzes that relationship and develops an improvement, rather than merely rephrasing the initial instructions.
    \item Let every revision consider successes, failures, and counterexamples.
    \item Keep each revision concise and self-contained, preserving the initial instructions' inputs and output format. Describe an analysis approach without supplying the knowledge it should discover.
\end{promptitems}

\promptsec{Output}

Return one JSON object containing one entry for each requested evolution view:

\begin{lstlisting}[style=promptjson]
{
  "prompts": [
    {
      "branch_id": "requested view ID",
      "strategy": "The complete revised analysis instructions.",
      "rationale": "Briefly explain how its analysis approach differs."
    }
  ]
}
\end{lstlisting}
\end{promptbox}

\clearpage
\begin{promptbox}[label={prompt:edit}]{Writing a knowledge candidate}
You write a short, standalone knowledge document from a knowledge proposal.

\promptsec{Your Task}

Turn the proposed guidance into a complete \texttt{SKILL.md} that addresses the identified failure mode. The document will be supplied directly to an agent and tested on its own.

\promptsec{Inputs}

\begin{promptitems}
    \item The knowledge proposal and its rationale.
    \item The experience supporting the proposal.
    \item The document length budget.
\end{promptitems}

\promptsec{Requirements}

\begin{promptitems}
    \item Implement the proposed guidance with concrete, usable instructions.
    \item Include the context and operational details needed to apply the method independently.
    \item Use a short example when it helps explain the method.
    \item Keep the document coherent and focused on the proposed improvement. Ground its details in the supplied experience.
    \item Stay within the document length budget.
\end{promptitems}

\promptsec{Output}

Return one JSON object:

\begin{lstlisting}[style=promptjson]
{
  "content": "# Title\n\nPurpose and complete instructions...",
  "rationale": "Briefly explain how the document implements the proposal."
}
\end{lstlisting}
\end{promptbox}

\clearpage
\begin{promptbox}[label={prompt:cons-propose}]{Proposing the knowledge update}
You select knowledge from replayed knowledge candidates to improve the current knowledge.

\promptsec{Your Task}

Identify useful guidance already present in the knowledge candidates. Use the paired trajectories to assess whether that guidance plausibly contributed to an improvement, then propose additions or corrections to the current knowledge.

A subsequent optimizer call will integrate your proposal into the current knowledge.

\promptsec{Inputs}

\begin{promptitems}
    \item The current knowledge document.
    \item Tasks selected for positive pairwise common gain with a strict majority of evolution views improving, their trajectories under the current knowledge and the improved replay trajectories, and the corresponding rewards and utility gains.
    \item The relevant knowledge candidates and their original proposals and rationales.
\end{promptitems}

\promptsec{Requirements}

\begin{promptitems}
    \item Start with specific guidance in each knowledge candidate. Then check whether the corresponding trajectory applied it at a relevant decision.
    \item A successful trajectory alone does not establish that the knowledge candidate supplied the successful method. If the method appears only in the trajectory, leave it out of this proposal.
    \item Compare supported guidance with the current knowledge. Preserve useful additions, corrections, explanations, and examples.
    \item Combine overlapping guidance while preserving its meaning and applicability. Do not broaden a rule beyond what its source supports.
    \item Guidance need not appear in every knowledge candidate. If no useful addition has a clear connection to both a knowledge candidate and its trajectory, return an empty proposal.
\end{promptitems}

\promptsec{Output}

Return one JSON object:

\begin{lstlisting}[style=promptjson]
{
  "proposal": "Concrete guidance to add or revise in the current knowledge.",
  "rationale": "Identify the source guidance in the knowledge candidates, the relevant trajectory behavior, and what the current knowledge needs to change."
}
\end{lstlisting}
\end{promptbox}

\clearpage
\begin{promptbox}[label={prompt:cons-edit}]{Integrating the knowledge update}
You integrate proposed knowledge into an existing knowledge document.

\promptsec{Your Task}

A candidate \texttt{SKILL.md} has been initialized from the current knowledge. Use Read, Edit, and Write to update that file using the supplied proposal.

Keep one coherent, complete document that preserves useful existing knowledge and incorporates the proposed improvements.

\promptsec{Inputs}

\begin{promptitems}
    \item The candidate \texttt{SKILL.md}, available through Read.
    \item The knowledge proposal and its rationale.
    \item The selected trajectories under the current knowledge and the improved replay trajectories.
\end{promptitems}

\promptsec{Requirements}

\begin{promptitems}
    \item Read the current document and implement the proposed additions or corrections in the relevant sections.
    \item Implement the supplied proposal without deriving additional guidance from the trajectories.
    \item Merge duplicate or overlapping logic into one coherent treatment. Preserve useful differences in explanations, operational details, and examples as supplements.
    \item When merging overlapping guidance, preserve meaningful differences in conditions, details, and examples; do not turn a qualified rule into an unconditional instruction.
    \item Before treating proposed content as already covered, check whether the document contains the relevant detail or explanation, not only a similar action instruction.
    \item Keep the document complete and the changes grounded in the proposal and supplied trajectories.
    \item Save your changes to \texttt{SKILL.md}. If the proposal provides no useful addition or correction, leave the file unchanged.
\end{promptitems}

\promptsec{Output}

Call Finish after editing, with a brief rationale explaining what was added, revised, merged, or left unchanged.

The runtime reads and saves the complete file; your final response does not need to repeat its contents.
\end{promptbox}

\clearpage
\begin{promptbox}[label={prompt:negative}]{Constructing the noisy knowledge $K^-$}
Summarize only the behavioral CHANGES in the supplied documents.

\promptsec{Your Task}

The two source evolution rounds did not produce a new accepted knowledge update. For each round, any supplied task has the highest positive common gain on negated utility gains, with a strict majority of evolution views degrading.

\promptsec{Inputs}

\begin{promptitems}
    \item Changed knowledge candidates.
    \item Their exact diffs against the unchanged current knowledge (the supplied original).
    \item Baseline and degrading replay evidence.
\end{promptitems}

\promptsec{Requirements}

\begin{promptitems}
    \item Compare each modified document against the current knowledge by MEANING, not just line diff. Rewriting or shortening an existing rule is not a new rule.
    \item Before including each summary, ask whether an agent already following the entire current knowledge would make a different decision because of this instruction. If not, omit it.
    \item For each distinct added or modified instruction, identify its source evolution view, cite its supplied \texttt{change\_id}. Summarize only that behavioral difference in \texttt{summary}, using direct instructions that stand alone.
    \item Preserve the meaning of changed guidance, including mistakes. Do not correct it, negate it, or turn it into advice about avoiding mistakes.
    \item Do not copy unchanged content of the current knowledge or infer guidance from pure deletions.
\end{promptitems}

\promptsec{Output}

The summaries will be joined directly into the noisy knowledge. Return JSON:

\begin{lstlisting}[style=promptjson]
{
  "changes": [
    {"change_id": "...", "summary": "Standalone instruction describing only the change"}
  ]
}
\end{lstlisting}

Return \texttt{\{"changes": []\}} when there is no identifiable added or modified guidance.
\end{promptbox}

%% file: Sections/appendix_coevolution.tex
\subsection{Design Considerations}
\label{app:coevo-design}

Knowledge evolution and model evolution depend on each other. The utility
of knowledge depends on the model that uses it: guidance that helps the
current model can become redundant or even harmful after the model is
trained. Model evolution, in turn, internalizes knowledge through
self-distillation from the model conditioned on it, which requires the
knowledge to remain fixed. Co-evolution therefore has two requirements:
\begin{itemize}
    \item \textbf{(R1)} knowledge is evaluated with the model that uses it;
    \item \textbf{(R2)} $K$ and $K^-$ remain fixed while the model
    internalizes them.
\end{itemize}
A further consideration is efficiency. Knowledge evolution typically requires
multiple rounds before it produces stable knowledge with clear utility
gains. If model evolution waits for these rounds, training stalls; if it
uses knowledge too early, it internalizes weak guidance. The two
procedures below handle this trade-off differently. The sequential
procedure waits and satisfies both requirements exactly, whereas the
asynchronous procedure relaxes R1 during knowledge evolution so that
training uses knowledge as soon as it proves useful.

\subsection{Sequential Co-Evolution}
\label{app:sequential-co-evolution}

The sequential procedure (Algorithm~\ref{alg:co-evolution}) alternates
the two processes. Knowledge evolution runs with the model fixed, so every
candidate is replayed and validated with the model that will use it. It
continues until $M$ consecutive updates are rejected, indicating that
knowledge evolution no longer improves the knowledge under the current
model. Model evolution then performs $T$ updates with $K$ and $K^-$ fixed.
The rejected rounds share the same model and knowledge, and thus provide a
common baseline for constructing $K^-$ from consistent utility degradation
(\S~\ref{sec:internalization}). Afterwards, knowledge evolution continues
from the current $K$ with experience collected by the updated model.

This procedure satisfies R1 and R2 exactly, but each process waits for the
other: training is idle during the multiple rounds of knowledge evolution,
and knowledge evolution is idle during training.

\begin{algorithm}[htbp]
\caption{\method{}: Sequential Co-Evolution}
\label{alg:co-evolution}
\small
\begin{algorithmic}[1]
\Require Policy $\pi_\theta$, rejection threshold $M$, model-update
budget $T$, total budget.
\State Initialize $K\gets\varnothing$.
\While{total budget remains}
    \State \textcolor{knowledgeblue}{\textbf{[Knowledge Evolution]}}
    \Repeat
        \State Collect experience $\mathcal{E}$ using $\pi_\theta$
        conditioned on $K$.
        \State Generate knowledge candidates $\{K^{(j)}\}_{j=1}^{N}$.
        \Comment{Eq.~\ref{eq:knowledge-views}}
        \State Replay tasks and compute common gains.
        \Comment{Eq.~\ref{eq:common-gain}}
        \State Construct $\widetilde K$ from the selected tasks.
        \Comment{Eq.~\ref{eq:knowledge-update}}
        \State Accept $K\gets\widetilde K$ if held-out validation
        performance improves.
    \Until{$M$ consecutive updates are rejected or the total budget is
    exhausted}
    \State \textcolor{modelred}{\textbf{[Model Evolution]}}
    \State Construct $K^-$ from consistent utility degradation in the
    last $M$ rejected rounds.
    \For{up to $T$ updates within the remaining budget}
        \State Collect trajectories with and without $K$.
        \State Update $\theta$ with $\mathcal{L}_{\mathrm{SPLICE}}$.
        \Comment{Eq.~\ref{eq:joint-loss}}
    \EndFor
\EndWhile
\State \Return $\pi_\theta$ and $K$.
\end{algorithmic}
\end{algorithm}

\subsection{Asynchronous Co-Evolution}
\label{app:asynchronous-co-evolution}

The asynchronous procedure (Algorithm~\ref{alg:async-co-evolution}) runs
knowledge evolution and model evolution at the same time. Knowledge
evolution keeps improving its own copy of knowledge, $K_{\mathrm{evo}}$,
using experience generated during training, while model evolution trains
with $K$, the most recent $K_{\mathrm{evo}}$ that improves held-out validation
performance with the current model. Useful knowledge thus enters training without waiting
for knowledge evolution to stop, and training does not stall. Before
training starts, knowledge evolution runs several rounds with the initial
model to obtain initial knowledge.

\paragraph{Consistent references for comparison.}
Training experience is generated under whichever knowledge the current
stage uses. Knowledge evolution therefore takes trajectories generated
under the same knowledge $K'$ and measures replay gains relative to $K'$.
Updates supported by common gain are merged into the latest
$K_{\mathrm{evo}}$, which can differ from $K'$; knowledge evolution
therefore accumulates improvements instead of restarting with each batch.

\paragraph{Adopting knowledge for training.}
Model evolution proceeds in training stages. At the start of each stage,
the latest $K_{\mathrm{evo}}$ is compared with no knowledge on the held-out
validation set using the current model. It is adopted as $K$ only if it
improves validation performance; otherwise, the stage trains with GRPO
alone. Each stage that trains with $K$ first constructs $K^-$ from the
replay results of recent knowledge evolution rounds by selecting tasks with
consistent utility degradation using $-\Delta_{ij}$
(\S~\ref{sec:internalization}). Because knowledge evolution evaluates
candidates with earlier models, validating with the current model restores R1 at the time of
adoption. Within a stage, $K$ and $K^-$ remain fixed, which satisfies R2.

\paragraph{When to replace the training knowledge.}
The batch reward gap, i.e., the mean of
$\overline R_i^{K}-\overline R_i^{\varnothing}$ over a batch, measures how
much of the benefit of $K$ the model has not yet internalized.
Self-distillation is applied only when this gap exceeds
$\epsilon_{\mathrm{sd}}$, which extends the task-level weight $w_i$ to the
batch level; other batches use GRPO alone. A stage ends once the gap falls
to $\epsilon_{\mathrm{end}}$ or below after at least $T_{\min}$ updates,
indicating that $K$ has been largely internalized, or after $T_{\max}$
updates, so that newer knowledge can be adopted. A stage without knowledge
ends after $T_{\min}$ updates. $T_{\min}$ keeps each stage long enough for
stable training, while $T_{\max}$ limits how long training relies on
outdated knowledge.

\paragraph{Restarting knowledge evolution.}
If no knowledge is adopted for $L$ consecutive stages, which
indicates that the knowledge accumulated in $K_{\mathrm{evo}}$ no longer
helps the current model, knowledge evolution restarts from empty knowledge
while keeping the current model and the collected experience. At the end
of training, the latest $K_{\mathrm{evo}}$ is validated with the final model
in the same way to determine the returned knowledge.

\begin{algorithm}[htbp]
\caption{\method{}: Asynchronous Co-Evolution}
\label{alg:async-co-evolution}
\small
\begin{algorithmic}[1]
\Require Policy $\pi_\theta$; reward-gap thresholds $\epsilon_{\mathrm{sd}}$
and $\epsilon_{\mathrm{end}}$; stage budgets $T_{\min}\le T_{\max}$;
restart patience $L$; total budget.
\State Run knowledge evolution with the initial model to obtain
$K_{\mathrm{evo}}$.
\State Run the two processes below at the same time.
\Statex \textcolor{knowledgeblue}{\textbf{[Knowledge Evolution]}}
\While{training continues}
    \State Take training experience generated under the same knowledge.
    \State Run knowledge evolution on it
    (\S~\ref{sec:knowledge-evolution}) and update $K_{\mathrm{evo}}$.
    \State Restart from $K_{\mathrm{evo}}\gets\varnothing$ if requested.
\EndWhile
\Statex \textcolor{modelred}{\textbf{[Model Evolution]}}
\While{total budget remains}
    \State Adopt $K\gets K_{\mathrm{evo}}$ if it improves validation
    performance with the current model; otherwise $K\gets\varnothing$.
    \State If $K\neq\varnothing$, construct $K^-$ from consistent utility
    degradation in recent rounds before training.
    \State Request a restart after $L$ consecutive stages without adopted
    knowledge.
    \Repeat
        \State Collect trajectories with and without $K$; pass them to
        knowledge evolution.
        \State Update $\theta$ with $\mathcal{L}_{\mathrm{SPLICE}}$ if the
        reward gap exceeds $\epsilon_{\mathrm{sd}}$; otherwise with
        $\mathcal{L}_{\mathrm{GRPO}}$.
    \Until{the gap is at most $\epsilon_{\mathrm{end}}$ after $T_{\min}$
    updates, or $T_{\max}$ updates are reached}
\EndWhile
\State \Return $\pi_\theta$, and $K_{\mathrm{evo}}$ if it improves validation
performance with the final model.
\end{algorithmic}
\end{algorithm}

\subsection{Discussion: Relaxing R1}
\label{app:coevo-discussion}

In the asynchronous procedure, the model is updated while knowledge
evolution runs, so replay can use a newer model than the one that
generated the experience. For task $x_i$, the replay gain then decomposes
as $\Delta_{ij}=\Delta^{*}_{ij}+b_i$, where $\Delta^{*}_{ij}$ is the gain
of $K^{(j)}$ measured with the replay model alone, and $b_i$ is the reward
change caused by the model update under the knowledge that generated the
experience. Since all candidates for $x_i$ are replayed with the same
model, $b_i$ is shared across views, and
\begin{equation}
    \widehat C_i
    =
    \frac{1}{\binom{N}{2}}\sum_{j<\ell}
    \min(\Delta^{*}_{ij},\Delta^{*}_{i\ell})
    + b_i .
    \label{eq:app-shifted-common-gain}
\end{equation}
The inconsistency penalty is unchanged, so the comparison among evolution
views remains valid. The shift instead changes which tasks pass the
threshold $\widehat C_i>0$: model improvement ($b_i>0$) admits tasks whose
gains come from training rather than from knowledge, and model degradation
($b_i<0$) excludes useful evidence. Proposition~\ref{prop:view-stability}
likewise assumes a fixed model.

Two design choices limit this effect. Knowledge evolution uses recent
experience, which keeps the model change between experience collection
and replay small. More importantly, knowledge enters training only after
it improves validation performance over no knowledge with the current
model, so an update whose
apparent gain comes from the model update is not adopted unless it also
helps the current model. The remaining cost is efficiency: a shifted
threshold can select less informative tasks and produce updates that are
later rejected. Compared with the sequential procedure, the asynchronous
procedure accepts this cost in exchange for using knowledge without
waiting for knowledge evolution to finish.
\newpage

%% file: Sections/appendix_related.tex
\begingroup
\paragraph{Knowledge evolution for LLM agents.}
LLM agents can improve without parameter updates by turning experience
into external knowledge. Early methods reflect on failures and store textual
lessons: Reflexion keeps verbal self-reflections for subsequent attempts
\citep{shinn2023reflexion}, ExpeL extracts insights from successful and failed
trajectories across tasks \citep{zhao2024expel}, CLIN maintains a textual
memory updated across trials \citep{majumder2023clin}, and Agent-Pro reflects
at the level of the agent's policy \citep{zhang2024agent}. Another line
accumulates procedural knowledge, such as executable skill libraries
\citep{wang2023voyager}, workflows induced from past trajectories
\citep{wang2024agent}, and skills that are automatically discovered and refined
\citep{alzubi2026evoskill,yang2026skillopt,wei2026evoharnesscontexttoharnessskillcompilation}, including learnable memory
operations \citep{zhang2026memskill}.
Other methods maintain structured memories or contexts that evolve during task solving, including agentic memory
organization \citep{xu2026mem}, adaptive test-time memory
\citep{suzgun2026dynamic}, reasoning strategies distilled from successes and
failures \citep{ouyang2026reasoningbank}, and evolving playbooks
\citep{zhang2026agentic}. Training-Free GRPO further distills experiential
knowledge from group-relative comparisons of rollouts
\citep{cai2025training}. Related studies examine the reliability of skills
and agentic reasoning: model-generated skills do not reliably
improve performance \citep{li2026skillsbench,huang2026raw,dong2026agent},
unchecked accumulation degrades skill retrieval \citep{zhang2026library}, and
agents can fail to use experience and external evidence reliably
\citep{xiong2026memory,zhao2026papermindbenchmarkingagenticreasoning}. Some methods therefore retain only
updates with positive estimated gain \citep{tang2026memory}.
\method{} focuses on which
evidence should drive a knowledge update, using utility consistency across
multiple evolution views of the same experience.

\paragraph{Knowledge internalization and co-evolution.}
For LLM agents, recent work internalizes skills into
the policy through reinforcement learning and self-distillation
\citep{lu2026self,he2026siri}, progressively reduces reliance on skills
provided in context \citep{lu2026skill0}, or attributes success to skills by
contrasting rollouts with and without them \citep{lin2026skillc}. Other work
evolves knowledge together with the policy: skill banks are updated during
online reinforcement learning \citep{xia2026skillrl}, skill creation is coupled
with policy optimization \citep{he2026reskill}, experience is distilled into
principles that guide further policy learning \citep{wu2025evolver}, and a
single policy learns to select, use, and distill skills
\citep{shi2026skill1}. On-policy self-distillation, in which the model
conditioned on privileged context serves as its own teacher, has become a
dense training signal for LLMs \citep{zhao2026self}. Since privileged context
can distort the teacher's token distribution, recent variants modulate
token-level guidance by teacher--student disagreement \citep{li2026demopsd},
reverse the distillation direction \citep{shen2026anti}, or
contrast signals under correct and wrong hints \citep{pan2026rlcsd}. Negative
signals have likewise been shown to provide useful supervision in
reinforcement learning for LLMs \citep{zhu2026surprising}. \method{} derives
its negative reference from knowledge evolution itself, using knowledge
associated with consistent utility degradation, and combines it with
task-level weighting and a token-level gate during internalization.

\paragraph{Consistency regularization.}
Consistency regularization encourages a model to make consistent predictions
for the same input under perturbations. In semi-supervised learning, early
methods enforce agreement under stochastic transformations, dropout, and
ensembling \citep{bachman2014learning,sajjadi2016regularization,laine2016temporal},
and later methods improve the targets with weight-averaged teachers
\citep{tarvainen2017mean} or strong data augmentation
\citep{xie2020unsupervised}. The idea has been extended to supervised training
and fine-tuning through consistency between dropout sub-models
\citep{wu2021r} or under parameter perturbations \citep{ni2024pace}.

\par
\endgroup
\newpage

%% file: Sections/appendix_results.tex
This section provides additional analyses of consistency regularization
and knowledge-model co-evolution, complementing Section~\ref{sec:analysis}.

\begin{figure}[H]
    \centering
    \includegraphics[width=\linewidth]{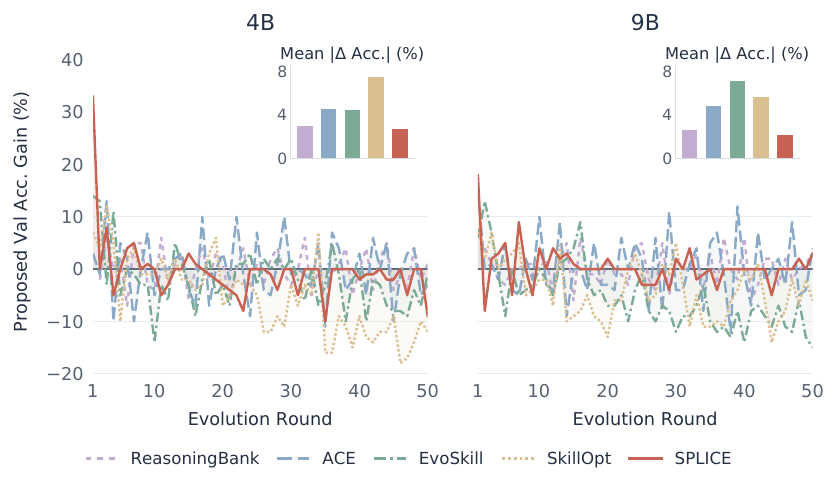}
    \caption{\textbf{Validation gains of proposed updates on WebShop.}
    Validation Acc. of the proposed update minus that of the current
    knowledge (in points) at each evolution round, including rejected updates.
    Insets show the mean absolute change.}
    \label{fig:consistency-gains}
\end{figure}

\paragraph{Stability of proposed updates.}
Figure~\ref{fig:consistency-gains} reports the change in validation Acc.
caused by the proposed update at each of the first 50 evolution rounds on
WebShop, including updates that are later rejected. The baselines frequently
propose updates that lower validation Acc. by 10 points or more, especially
SkillOpt and EvoSkill. In contrast, the updates proposed by \method{} mostly
stay close to zero or positive and have the smallest mean absolute change for
both model sizes. This is consistent with the design of common gain, which
discounts evidence whose gains appear under only some evolution views and thus
avoids updates that change the knowledge in unstable ways.

\begin{figure}[H]
    \centering
    \includegraphics[width=0.85\linewidth]{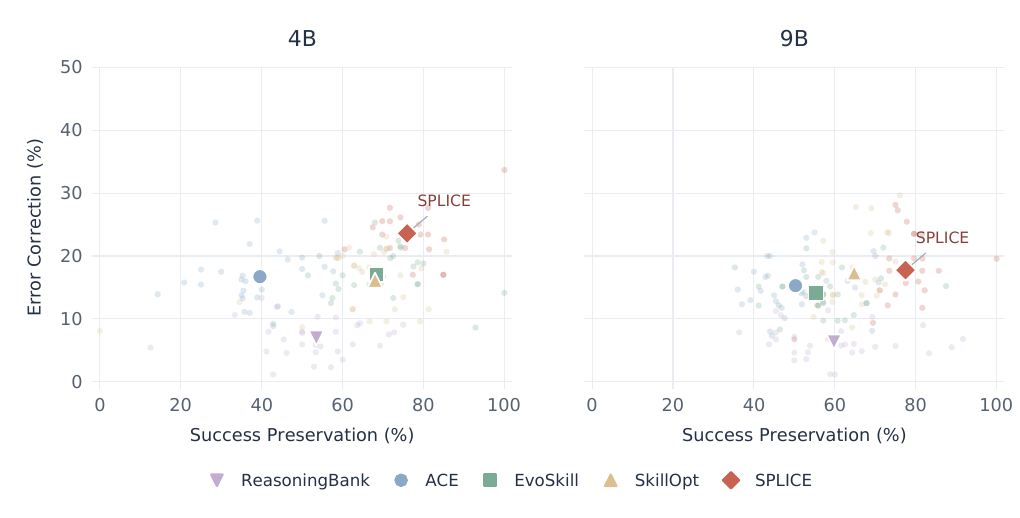}
    \caption{\textbf{Error correction and success preservation on WebShop.}
        Percentages of tasks failed under the current knowledge that a
        proposed update solves, and of solved tasks that remain solved.
        Small points denote proposed updates in the first 30 rounds; large
        markers show pooled rates.}
    \label{fig:consistency-repair-retention}
\end{figure}

\paragraph{Error correction and success preservation.}
To examine how proposed updates affect individual tasks, we compare each
proposed update with the current knowledge on the validation tasks. Error
correction is the percentage of tasks failed under the current knowledge that
the proposed update solves, and success preservation is the percentage of
tasks solved under the current knowledge that remain solved.
Figure~\ref{fig:consistency-repair-retention} reports both metrics for the
proposed updates in the first 30 evolution rounds on WebShop, excluding rounds
without an update. \method{} attains the highest pooled rates for both
metrics: 23.60\% error correction and 75.94\% success preservation with
Qwen3.5-4B, and 17.75\% and 77.47\% with Qwen3.5-9B, whereas the best baseline
preservation rates are 68.40\% and 64.79\%. Updates selected by common gain
thus correct more errors while breaking fewer existing successes.

\begin{figure}[H]
    \centering
    \includegraphics[width=\linewidth]{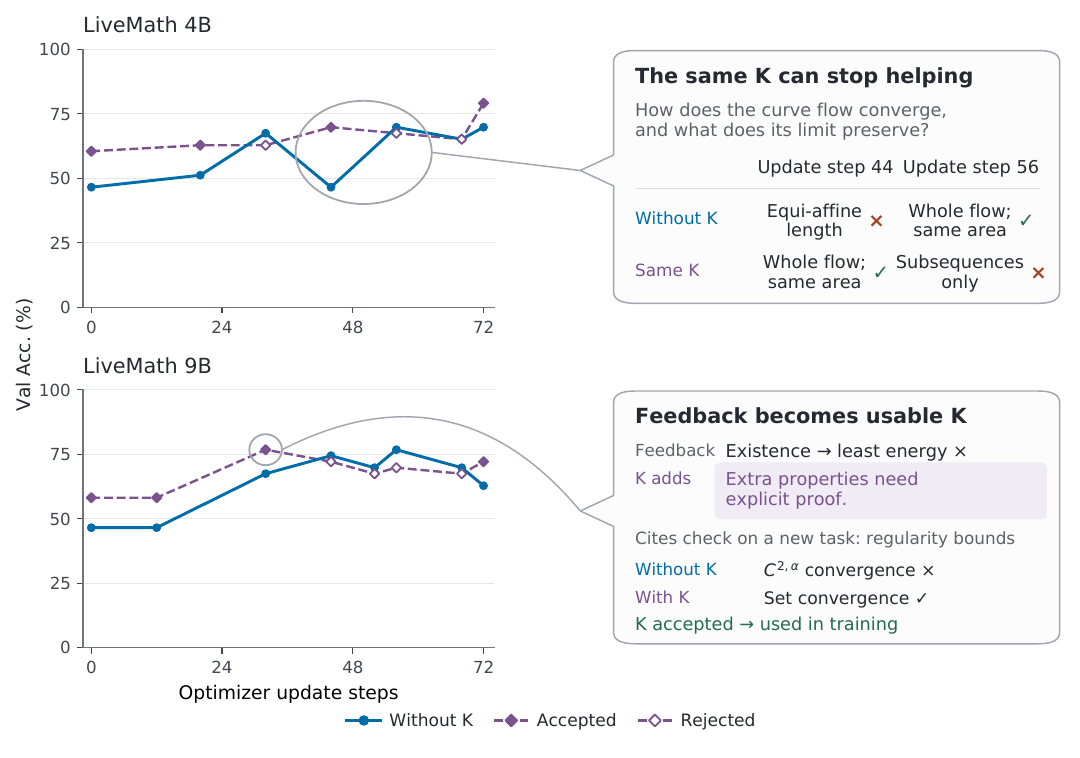}
    \caption{\textbf{Training dynamics and knowledge use on LiveMath.}
        Left: validation Acc. over optimizer update steps without knowledge and
        with the latest knowledge, which is adopted (filled) or rejected
        (open) by validation with the current model. Right: the same
        knowledge helps the 4B model at update step 44, but leads to an error at
        update step 56 when the model already answers correctly without it;
        the 9B model uses a check learned from execution feedback to
        correct an unsupported convergence claim, and the knowledge is
        accepted for subsequent training.}
    \label{fig:training-coevolution}
\end{figure}

\paragraph{Training dynamics of co-evolution.}
Figure~\ref{fig:training-coevolution} (left) tracks validation Acc. on
LiveMath during co-evolution, evaluating the model without knowledge and with
the latest knowledge at each optimizer update step. Early in training, the knowledge
improves the model substantially. As the model improves without knowledge,
this benefit shrinks, and some knowledge updates no longer improve the current
model and are rejected. The examples on the right illustrate both directions
of this interplay. For the 4B model, the same knowledge helps the model answer
correctly at update step 44, but by update step 56 the model answers correctly without
it, whereas adding the knowledge leads to an error. For the 9B model, a check
learned from execution feedback corrects an unsupported convergence claim on
another task: at update step 32, the model answers incorrectly without knowledge
and correctly with it, and the knowledge is accepted for subsequent training.
These dynamics motivate validating knowledge with the current model before
using it for training (Appendix~\ref{app:co-evolution}).

\input{Tables/analysis-csd}

\begin{samepage}
\paragraph{Contrastive self-distillation.}
\begingroup
To examine the contribution of CSD, we compare no self-distillation
($\lambda=0$), self-distillation without $K^-$ ($\gamma_t=0$), CSD without
the token-level gate ($\gamma_t=1$), and full CSD.
Full CSD achieves the best accuracy in all four reported co-evolution
settings (Table~\ref{tab:csd-objectives}).
Removing the gate generally lowers accuracy, suggesting that $K^-$ can still
contain useful guidance and should not be used as an unconditional negative reference.
This is consistent with Section~\ref{sec:internalization}: $K^-$ complements
the guidance from $K$, and the gate preserves the useful guidance they share.\par
\endgroup
\end{samepage}

%% file: Tables/analysis-csd.tex
\begin{table}[!t]
    \color{black}
    \arrayrulecolor{black}
    \centering
    \captionsetup{font=small,skip=4pt,labelfont={color=black},textfont={color=black}}
    \caption{\textbf{CSD objectives during co-evolution.}
    Test Acc. (\%), averaged over three evaluations per trained model.
    All self-distillation variants use $\lambda=0.3$.}
    \label{tab:csd-objectives}
    \small
    \setlength{\tabcolsep}{6pt}
    \renewcommand{\arraystretch}{1.05}
    \begin{tabular}{@{}lcccc@{}}
        \toprule
         & \multicolumn{2}{c}{\textbf{WebShop}} & \multicolumn{2}{c}{\textbf{LiveMath}} \\
        \cmidrule(lr){2-3}\cmidrule(l){4-5}
        \textbf{Objective} & 9B & 4B & 9B & 4B \\
        \midrule
        No self-distillation & $66.13$ & $75.80$ & $64.24$ & $65.67$ \\
        Without $K^-$ & $73.40$ & $74.93$ & $66.23$ & $65.89$ \\
        Without token gate & $70.80$ & $61.53$ & $63.80$ & $66.23$ \\
        Full CSD & $\mathbf{76.13}$ & $\mathbf{78.67}$ & $\mathbf{68.87}$ & $\mathbf{67.44}$ \\
        \bottomrule
    \end{tabular}
    \arrayrulecolor{black}
\end{table}

%% file: Sections/appendix_cases.tex
\begingroup
\input{Sections/appendix_case_style}
\setlength{\parindent}{0pt}
\color{csink}

\subsection{Evolved Knowledge}
\label{app:evolved-knowledge-case}
\begin{csbanner}
Turning a search loop into a knowledge update
\end{csbanner}

\cstask{Find a loose-fitting women's top in d1-hot pink, size large, for under 50 dollars.}{WebShop\;\textperiodcentered\;Qwen3.5-9B}

\cssection{01}{Original knowledge}{Related excerpts from $K$}
\begin{csquote}{cspale}{csblue}
Required color/size confirmed absent from product page $\to$ \textbf{Go back to search and try a different product}
\par\vspace{3pt}
If 2--3 different search queries all return the same result set (same items appear regardless of how you reword the query), the index does not contain a better match. \textbf{Stop searching.}
\end{csquote}

\cssection{02}{Task execution}{}
\noindent

\begin{minipage}[t]{0.487\linewidth}
\begin{cstrace}{With existing $K$}{15}
\begin{minipage}[t][154pt][t]{\linewidth}
\fontsize{8.1}{10}\selectfont
\csaction{01}{search[...]}{white}{csink}
\csaction{02}{click[Product A]}{csredbg}{csred}
\csaction{03}{click[back to search]}{white}{csink}
\csaction{04}{search[...]}{white}{csink}
\csaction{05}{click[back to search]}{white}{csink}
\csaction{06}{search[...]}{white}{csink}
\csaction{07}{[no valid action]}{white}{csmuted}
\csaction{08}{click[Product A]}{csredbg}{csred}
\csaction{09}{click[back to search]}{white}{csink}
\csaction{10}{search[...]}{white}{csink}
\csaction{11}{click[Product A]}{csredbg}{csred}
\csaction{12}{click[back to search]}{white}{csink}
\csaction{13}{search[...]}{white}{csink}
\csaction{14}{click[Product A]}{csredbg}{csred}
\csaction{15}{click[a1-wine red]}{white}{csink}
\end{minipage}
\csoutcome{0.00}{csred}{Product A lacks the requested color.\\Four visits; no purchase.}
\end{cstrace}
\end{minipage}
\hfill
\begin{minipage}[t]{0.487\linewidth}
\begin{cstrace}{With candidate 2}{15}
\begin{minipage}[t][154pt][t]{\linewidth}
\fontsize{8.1}{10}\selectfont
\csaction{01}{search[...]}{white}{csink}
\csaction{02}{[no valid action]}{white}{csmuted}
\csaction{03}{click[Product C]}{white}{csink}
\csaction{04}{click[back to search]}{white}{csink}
\csaction{05}{search[...]}{white}{csink}
\csaction{06}{click[Product A]}{csredbg}{csred}
\csaction{07}{click[back to search]}{white}{csink}
\csaction{08}{search[...]}{white}{csink}
\csaction{09}{click[Product B]}{csgreenbg}{csgreen}
\csaction{10}{click[d1-hot pink]}{white}{csink}
\csaction{11}{click[large]}{white}{csink}
\csaction{12}{[no valid action]}{white}{csmuted}
\csaction{13}{click[d1-hot pink]}{white}{csink}
\csaction{14}{click[large]}{white}{csink}
\csaction{15}{click[buy now]}{csgreenbg}{csgreen}
\end{minipage}
\csoutcome{1.00}{csgreen}{Product B offers the requested color.\\Color and size selected; purchase succeeds.}
\end{cstrace}
\end{minipage}

\cssection{03}{Candidate knowledge across views}{Common gain = 1.00}
\cscandidate{Candidate 1}{1.00}{8}{csgreen}{%
Coded color absent, visited same product 2+ times AND searches returning same results $\to$ \textbf{Stop searching. Select closest color by hue family.} Proceed to size + buy now.}
\cscandidate{Candidate 2}{1.00}{15}{csgreen}{%
The wrong choice is to \textbf{loop back to the same product page repeatedly} [\ldots] Pick a \textbf{different candidate} from the results page}
\cscandidate{Candidate 3}{1.00}{13}{csgreen}{%
\textbf{The saturation rule overrides the color-absent rule.} [\ldots] A different product on the current results page is more likely to carry the exact color $\to$ \textbf{click that other product first before approximating.}}
\csnote{Different candidate instructions all raise replay reward from 0 to 1.}

\cssection{04}{Consolidated knowledge}{Abridged}
\begin{csquote}{csgreenbg}{csgreen}

\csrule{1}{If the exact color is absent, inspect a different product from the same results page.}
\csrule{2}{Do not revisit a product after confirming that its required color is absent.}
\csrule{3}{Try another product before using a closest-color fallback.}
\end{csquote}
\csfooter{Accepted}{csgreen}{Validation Acc. $45\%\to49\%$}{Retained in final $K$}

\clearpage
\subsection{Noisy Knowledge}
\label{app:noisy-knowledge-case}
\begin{csbanner}
Constructing noisy knowledge from degrading candidates
\end{csbanner}
\cstask{Find a machine-washable shirt in dark heather, men's fit, size 2t, for under 40 dollars.}{WebShop\;\textperiodcentered\;Qwen3.5-9B}

\cssection{01}{Original knowledge}{Related excerpts from $K$}
\begin{csquote}{cspale}{csblue}
After clicking all N required options, immediately click \texttt{buy now} on the very next turn.
\par\vspace{3pt}
Use your recent action history (last 2 actions visible) to determine which required options have already been clicked this visit. \textbf{Once all required options appear in your recent history}, click \texttt{buy now} immediately---do not re-click any option.
\end{csquote}

\cssection{02}{Task execution}{}
\noindent

\begin{minipage}[t]{0.487\linewidth}
\begin{cstrace}{With existing $K$}{13}
\begin{minipage}[t][154pt][t]{\linewidth}
\fontsize{8.1}{10}\selectfont
\csaction{01}{search[...]}{white}{csink}
\csaction{02}{click[Product A]}{white}{csink}
\csaction{03}{click[dark heather]}{csgreenbg}{csgreen}
\csaction{04}{click[men]}{csgreenbg}{csgreen}
\csaction{05}{click[2t]}{csgreenbg}{csgreen}
\csaction{06}{click[dark heather]}{white}{csink}
\csaction{07}{click[men]}{white}{csink}
\csaction{08}{click[2t]}{white}{csink}
\csaction{09}{click[dark heather]}{white}{csink}
\csaction{10}{click[men]}{white}{csink}
\csaction{11}{click[2t]}{white}{csink}
\csaction{12}{click[dark heather]}{white}{csink}
\csaction{13}{click[buy now]}{csgreenbg}{csgreen}
\end{minipage}
\csoutcome{1.00}{csgreen}{All three options selected.\\Purchase succeeds.}
\end{cstrace}
\end{minipage}
\hfill
\begin{minipage}[t]{0.487\linewidth}
\begin{cstrace}{With candidate 3}{15}
\begin{minipage}[t][154pt][t]{\linewidth}
\fontsize{8.1}{10}\selectfont
\csaction{01}{search[...]}{white}{csink}
\csaction{02}{click[Product B]}{white}{csink}
\csaction{03}{click[back to search]}{white}{csink}
\csaction{04}{[no valid action]}{csredbg}{csmuted}
\csaction{05}{[no valid action]}{csredbg}{csmuted}
\csaction{06}{[no valid action]}{csredbg}{csmuted}
\csaction{07}{search[...]}{white}{csink}
\csaction{08}{click[Product B]}{white}{csink}
\csaction{09}{search[...]}{white}{csink}
\csaction{10}{search[...]}{white}{csink}
\csaction{11}{[no valid action]}{csredbg}{csmuted}
\csaction{12}{[no valid action]}{csredbg}{csmuted}
\csaction{13}{[no valid action]}{csredbg}{csmuted}
\csaction{14}{[no valid action]}{csredbg}{csmuted}
\csaction{15}{[no valid action]}{csredbg}{csmuted}
\end{minipage}
\csoutcome{0.00}{csred}{Stalls during search.\\Budget exhausted; no purchase.}
\end{cstrace}
\end{minipage}

\cssection{03}{Candidate knowledge across views}{Common gain on $-\Delta$ = 1.00}
\cscandidate{Candidate 1}{0.00}{15}{csred}{%
After you have clicked \textbf{N options across N consecutive turns}, click \texttt{buy now} on the very next turn---\textbf{regardless of what the history window shows}.}
\cscandidate{Candidate 2}{0.00}{15}{csred}{%
\textbf{Count each option click} you make on this page, starting from 1. The moment your count equals the number of required options, click \texttt{buy now} on the very next action.}
\cscandidate{Candidate 3}{0.00}{15}{csred}{%
When the goal specifies exactly 3 required options AND your recent history shows exactly 2 of them: \textbf{all 3 are already done}---the third fell out of the 2-turn window. Click \texttt{buy now} immediately.}

\cssection{04}{Noisy knowledge}{$K^-$ from candidate 3}
\begin{csquote}{csredbg}{csred}
When the goal specifies exactly 3 required options and recent history shows exactly 2 of them, \textbf{treat all 3 as done} (the third fell out of the 2-turn window) and click \texttt{buy now} immediately.
\end{csquote}
\csnote{Two visible selections do not establish that the third option was selected.}
\endgroup

%% file: Sections/appendix_case_style.tex
\definecolor{csink}{HTML}{24384B}
\definecolor{csmuted}{HTML}{667487}
\definecolor{csline}{HTML}{DCE3E9}
\definecolor{csblue}{HTML}{577F9F}
\definecolor{csgreen}{HTML}{568D76}
\definecolor{csred}{HTML}{C86354}
\definecolor{cspale}{HTML}{F3F5F7}
\definecolor{csgreenbg}{HTML}{EAF3EF}
\definecolor{csredbg}{HTML}{FAEFEC}

\newcommand{\csbody}{\fontsize{8.6}{10.4}\selectfont}
\newcommand{\csmeta}{\sffamily\fontsize{7.3}{8.7}\selectfont\color{csmuted}}
\tcbset{csbase/.style={
    enhanced, boxrule=0.4pt, sharp corners, boxsep=0pt,
    colback=white, colframe=csline,
    left=5pt, right=5pt, top=4pt, bottom=4pt,
    before skip=0pt, after skip=0pt, fontupper=\csbody
}}
\newtcolorbox{csbanner}{
    csbase, colback=csink, colframe=csink,
    left=7pt, right=7pt, top=5pt, bottom=5pt,
    fontupper=\sffamily\bfseries\fontsize{9.6}{11.3}\selectfont,
    colupper=white
}
\newtcolorbox{csquote}[2]{
    csbase, colback=#1, colframe=#1, boxrule=0pt,
    borderline west={1.5pt}{0pt}{#2},
    left=7pt, top=4pt, bottom=4pt
}
\newtcolorbox{cstrace}[2]{
    csbase, left=4pt, right=4pt, top=4pt, bottom=5pt,
    colbacktitle=cspale, coltitle=csink,
    fonttitle=\sffamily\bfseries\fontsize{8.5}{10}\selectfont,
    toptitle=4pt, bottomtitle=4pt,
    title={#1\hfill{\normalfont\csmeta #2 turns}}
}
\newcommand{\cssection}[3]{%
    \par\vspace{5pt}\noindent
    {\sffamily\bfseries\fontsize{8.5}{10}\selectfont
     \textcolor{csblue}{#1}\hspace{5pt}#2\hfill{\normalfont\csmeta #3}}%
    \par\nobreak\vspace{3pt}%
}
\newcommand{\cstask}[2]{%
    \vspace{4pt}
    \begin{tcolorbox}[csbase,colback=cspale,colframe=cspale,boxrule=0pt]
        {\sffamily\bfseries\fontsize{7.3}{8.7}\selectfont
         \textcolor{csmuted}{TASK (ABRIDGED)}}\hspace{6pt}#1\par\vspace{4pt}
        {\csmeta #2}
    \end{tcolorbox}%
}
\newcommand{\csaction}[4]{%
    \par\nointerlineskip
    \begingroup\setlength{\fboxsep}{0pt}%
    \hbox{\colorbox{#3}{\makebox[\linewidth][l]{%
        \rule[-3pt]{0pt}{10pt}%
        \makebox[17pt][l]{{\sffamily\fontsize{7.1}{10}\selectfont\color{csmuted}#1}}%
        {\ttfamily\color{#4}#2}}}}%
    \endgroup
}
\newcommand{\csoutcome}[3]{%
    \par\vspace{2pt}{\color{csline}\hrule}\vspace{4pt}
    {\sffamily\bfseries\fontsize{9}{10.5}\selectfont
     \textcolor{#2}{Reward #1}}\par\vspace{3pt}
    {\csmeta #3}%
}
\newcommand{\csnote}[1]{%
    \par\vspace{3pt}\noindent{\csmeta #1}\par
}
\newcommand{\cscandidate}[5]{%
    \begin{tcolorbox}[csbase,colback=cspale,left=5pt,right=5pt,top=4pt,bottom=4pt]
        \begin{minipage}[t]{0.19\linewidth}\vspace{0pt}
            {\sffamily\bfseries\fontsize{8}{9.4}\selectfont #1}\par\vspace{2pt}
            {\sffamily\fontsize{7.5}{9}\selectfont\textcolor{#4}{Reward #2}}\par
            {\csmeta #3 turns}
        \end{minipage}\hfill
        \begin{minipage}[t]{0.78\linewidth}\vspace{0pt}
            \csbody #5
        \end{minipage}
    \end{tcolorbox}
    \par\vspace{2pt}%
}
\newcommand{\csrule}[2]{%
    \noindent\hangindent=12pt\hangafter=1
    \makebox[12pt][l]{\sffamily\bfseries\textcolor{csgreen}{#1}}#2\par
    \vspace{2pt}%
}
\newcommand{\csfooter}[4]{%
    \par\vspace{6pt}{\color{csline}\hrule}\vspace{5pt}\noindent
    {\sffamily\bfseries\fontsize{8.5}{10}\selectfont\textcolor{#2}{#1}}%
    \hspace{8pt}{\csbody #3}\hfill{\csmeta #4}\par
}